\documentclass{article}
\usepackage{PRIMEarxiv}
\usepackage{color}
\usepackage{bm}
\usepackage{mfirstuc}
\usepackage{changepage}
\usepackage[utf8]{inputenc} 
\usepackage[T1]{fontenc}    
\usepackage{multirow}
\usepackage{wrapfig}
\usepackage{xr}
\usepackage{placeins}
\usepackage{xr-hyper}
\usepackage{hyperref}       
\usepackage{url}            
\usepackage{booktabs}       
\usepackage{amsfonts}       
\usepackage{nicefrac}       
\usepackage{microtype}      
\usepackage[normalem]{ulem}
\usepackage[caption=false]{subfig}
\usepackage{amsmath}
\usepackage{amssymb}
\usepackage{mathtools}
\usepackage{amsthm}
\usepackage{enumitem}
\usepackage[capitalize,noabbrev]{cleveref}
\usepackage[svgnames]{xcolor}
\usepackage{cancel}
\usepackage[textsize=tiny]{todonotes}
\usepackage{soul}
\usepackage[sort&compress]{natbib}
\usepackage[symbol]{footmisc}
\usepackage{titlesec}
\usepackage{makecell}
\makeatletter
\@addtoreset{section}{part}
\makeatother
\titleformat{\part}[display]
{\normalfont\LARGE\centering}{}{2pt}{}
\theoremstyle{plain}
\newtheorem{theorem}{Theorem}[section]

\newtheorem{lemma}[theorem]{Lemma}

\theoremstyle{definition}

\theoremstyle{remark}
\newtheorem{remark}[theorem]{Remark}

\DeclareMathOperator*{\argmax}{arg\,max}

\def\bs{\boldsymbol}
\definecolor{red}{rgb}{1,0,0}

\definecolor{blue}{rgb}{0,0,1}

\theoremstyle{plain}
\theoremstyle{definition}
\theoremstyle{remark}

\def\bs{\boldsymbol}
\definecolor{red}{rgb}{1,0,0}

\definecolor{blue}{rgb}{0,0,1}

\makeatletter

\newcommand*{\addFileDependency}[1]{
  \typeout{(#1)}
  \@addtofilelist{#1}
  \IfFileExists{#1}{}{\typeout{No file #1.}}
}
\makeatletter
\@addtoreset{section}{part}
\makeatother

\def\capmystringaux#1#2\relax{\uppercase{#1}\lowercase{#2}}

\begin{document}
\title{Nonparametric Automatic Differentiation Variational Inference with Spline Approximation
}

\author{
  Yuda Shao \\
  Department of Statistics  \\
  University of Virginia \\
  \texttt{ys7zg@virginia.edu} \\
   \And
  Shan Yu\textsuperscript{\dag} \\
  Department of Statistics\\
  University of Virginia \\
  \texttt{sy5jx@virginia.edu} \\
   \And
  Tianshu Feng\textsuperscript{\dag} \\
  Department of Systems Engineering and Operations Research \\
  George Mason University \\
  \texttt{tfeng@gmu.edu} \\
}
\footnotetext{\textsuperscript{\dag}Corresponding authors.}
\setcounter{footnote}{0}
\renewcommand{\thefootnote}{\arabic{footnote}}
\maketitle
\begin{abstract}
  Automatic Differentiation Variational Inference (ADVI) is efficient in learning probabilistic models. Classic ADVI relies on the parametric approach to approximate the posterior. In this paper, we develop a spline-based nonparametric approximation approach that enables flexible posterior approximation for distributions with complicated structures, such as skewness, multimodality, and bounded support. Compared with widely-used nonparametric variational inference methods, the proposed method is easy to implement and adaptive to various data structures. By adopting the spline approximation, we derive a lower bound of the importance weighted autoencoder and establish the asymptotic consistency. Experiments demonstrate the efficiency of the proposed method in approximating complex posterior distributions and improving the performance of generative models with incomplete data. 
\end{abstract}
\section{\uppercase{Introduction}}
\label{SEC:intro}

Variational Inference (VI) is widely used in data representation \citep{kingma_auto-encoding_2014, zhang2018advances}, graphical models \citep{wainwright2008graphical}, among others. VI approximates intractable distributions by minimizing the divergence between the true posterior and a chosen distribution family, aiming to identify an optimal distribution within this family. Unlike methods like Markov chain Monte Carlo (MCMC) sampling, VI is recognized for its computational efficiency and explicit distribution form \citep{blei2017variational}. Contemporary VI-based methods such as variational autoencoder (VAE) \citep{kingma_auto-encoding_2014} have garnered interest for learning representations of complex, high-dimensional data across fields like bioinformatics \citep{kopf2021mixture}, geoscience \citep{chen2022spatiotemporal}, and finance \citep{bergeron2022variational}.

Automatic Differentiation Variational Inference (ADVI) \citep{kucukelbir2017automatic} is a popular approach to derive variational inference algorithms for complex probabilistic models. Classic ADVI methods often adopt a parametric approach, approximating intractable posterior distributions with distributions from a specific probability distribution family (e.g., Gaussian distribution). However, it is limited to distributions allowing the reparametrization trick to calculate and backpropagate the gradient of the joint likelihood. Additionally, misspecified parametric assumptions can impair ADVI's efficacy to handle multimodal or skewed posteriors. 

Recent studies show that more flexible posterior approximations usually result in better performance \citep{kobyzev2020normalizing,pmlr-v51-han16}. In this paper, we aim to design a new type of variational inference based on spline approximation, named Spline Automatic Differentiation Variational Inference (S-ADVI), to improve the flexibility of posterior approximation while being interpretable. Spline approximation is an effective nonparametric tool for density estimation \citep{gu1993smoothing}. Theoretically, an arbitrary smooth density function can be well-approximated via weighted summation of a given sequence of spline bases. The shapes of spline bases are pre-specified and fixed, and the shapes of posterior distributions can be uniquely represented via the vector of spline coefficients. This property allows the assessment of the structure of posterior distributions and the interpretation of the latent representations. Consequently, the proposed S-ADVI achieves a balance of flexibility and parsimony. 

The proposed S-ADVI holds several merits over existing nonparametric methods.  First, a major limitation of nonparametric variational methods is the difficulty in providing theoretical guidance to recover the true posterior distribution. This paper theoretically investigates the asymptotic properties of the importance weighted autoencoder (IWAE), as well as the Kullback–Leibler (KL) divergence of spline approximations from the true posterior. Second, contrasted with other ADVI-based methods requiring pre-specified transformation to approximate distributions with bounded support, our approach simultaneously estimates the distribution support boundary. This adaptive boundary implementation enhances the accuracy and robustness of the model, ensuring superior performance. Last, the streamlined structure of the proposed method facilitates straightforward implementation and adaptability to various data structures.

In summary, our major contributions are:
(i) We design a novel nonparametric variational inference framework, S-ADVI, based on spline approximation to improve the flexibility of posterior approximation; (ii) Theoretical properties are established on the lower bound of IWAE and variational approximation errors of the proposed S-ADVI method; (iii) S-ADVI represents posterior distributions with deterministic vectors, allowing the assessment of shapes of posterior distributions and the interpretation of latent representation.    

\section{\uppercase{Background}}
\subsection{Variational Inferences}

Let $\bm x$ be the observed variables and $\bm z$ be the latent variables. We consider the joint distribution $p_{\theta}(\bm x, \bm z)$ for some parameter $\theta$, the generative model defined over the variables \citep{kingma_auto-encoding_2014}.  Learning $\theta$ typically requires the maximization of the marginal distribution of $\bm x$: 
$ p_\theta(\bm x) = \int p_\theta(\bm x | \bm z)p(\bm z) d \bm z,$
where $p(\bm z)$ is the prior of $\bm z$, and $p_\theta(\bm x|\bm z)$ is the conditional distribution of $\bm x$ given $\bm z$. 

Generally, the marginal likelihood function is intractable for flexible generative models. In VI, one common solution is to approximate the posterior $p(\bm z|\bm x)$ using a variational distribution $q_\phi(\bm z|\bm x)$ with $\phi$ being a collection of unknown parameters depending on the observed data $\bm x$. The problem then is transformed into maximizing the evidence lower bound (ELBO) $\mathcal{L}_{\text{ELBO}}\{\phi(\bm x)\}$ \citep{blei2017variational}:
\begin{align*}
&  \log p_\theta(\bm x) \ge  \mathcal{L}_{\text{ELBO}}\{\phi(\bm x)\},
\text{ where } \mathcal{L}_{\text{ELBO}}\{\phi(\bm x)\} \triangleq \mathbb{E}_{\bm z \sim q_\phi(\bm z| \bm x)} \left[\log \frac{p_\theta(\bm x |\bm z) p(\bm z)}{q_\phi(\bm z|\bm x)}\right].
\end{align*}

A tighter bound derived from importance weighting, namely importance-weighted autoencoder (IWAE) \citep{DBLP}, is then proposed with a strictly tighter log-likelihood lower bound than ELBO. To calculate an importance-weighted estimate of the log-likelihood,  $T$ independent samples are drawn from the posterior $\left\{\bm z_t\right\}_{t=1}^T \sim q_\phi(\bm z| \bm x)$, and a lower bound is then calculated as the log of the average of the ratio of the joint distribution and posterior for each sample:
\begin{align}
  \label{EQU:IWAE} 
   \mathcal{L}_{\text{IWAE}}\{\phi(\bm x)\} 
\triangleq \mathbb{E}_{\left\{\bm z_t \sim q_\phi(\bm z| \bm x)\right\}_{t=1}^T}\left[\log \frac{1}{T} \sum_{t=1}^T \frac{p_\theta\left(\bm x|\bm z_t\right) p\left(\bm z_t\right)}{q_\phi\left(\bm z_t| \bm x\right)}\right].
\end{align}
While tighter bound may not always be the best option \citep{rainforth2018tighter}, previous works suggest that multiple samples from the posterior help IWAE to be well-adapted to multimodal distributions and to approximate complex posteriors \citep{DBLP, morningstar2021automatic}. In this paper, we adopt IWAE as the objective function for its overall better properties. 

\subsection{Spline Approximation}
\label{SEC:Spline-App}

\begin{figure*}[!tb] 
    \centering
    \includegraphics[width=0.95\textwidth]{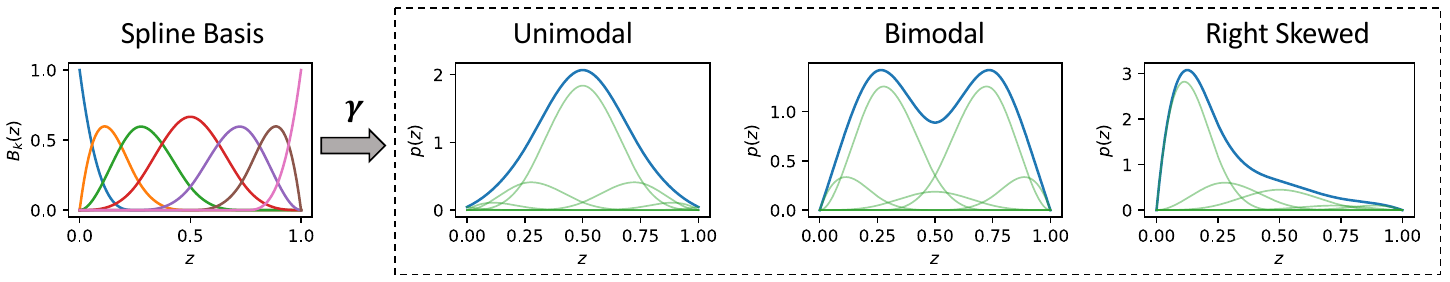}
    \vskip -0.1in
	\caption{\capitalisewords{illustration of the density functions based on different linear combinations of spline basis functions}} \label{fig:spline}
\end{figure*}

Spline approximation provides a method for approximating complex curves with a modest set of parameters, thereby achieving computational efficiency and having been widely used in machine learning and statistical learning areas, including generalized additive models \citep{hastie2017generalized}, functional data analysis \citep{fda-review}, longitudinal data analysis \citep{anderson1995smoothing}, neural networks \citep{pmlr-v80-balestriero18b,fakhoury2022exsplinet}, and point process intensity estimation \citep{loaiza-ganem2019deep}.

Let $\bs{\upsilon}$ be a partition of the interval $\mathcal{T} = [\upsilon_0, \upsilon_{H+1}]$ with $H$ interior knots, where $\bs{\upsilon}=\left\{\upsilon_{0}<\upsilon_{1}<\cdots <\upsilon_{H}<\upsilon _{H+1}\right\}$. Any spline function $s(z)$ within the spline space $\mathcal{U}$ of order $\varrho+1$ satisfies that: 1) the function $s(z)$ is a polynomial function with $\varrho$-degree (or less) on intervals $[\upsilon _{h},\upsilon_{h+1})$, $h=0,\ldots ,H$ and $\left[ \upsilon_{H},\upsilon_{H+1}\right]$; 2) it has $\varrho-1$ continuous derivatives over the entire region $\mathcal{T}$. Consider $\{B_{1,\mathcal{T}}(z),\ldots ,B_{K,\mathcal{T}}(z)\}^{\top}$ as a vector of the spline basis functions with degree $\varrho$ and partition $\bs{\upsilon}$, where $K = H + \varrho+1$. For the sake of notation simplicity,  for the rest of the paper, we define the normalized spline basis by $b_{k,\mathcal{T}}(z) \triangleq B_{k,\mathcal{T}}(z)/a_{k,\mathcal{T}}$, where $a_{k,\mathcal{T}} = \int_{\mathcal{T}}B_{k,\mathcal{T}}(z) dz$. It implies that $\int_{\mathcal{T}} b_{k,\mathcal{T}}(z) dz = 1$. Let $\bm b_\mathcal{T}(z)=\{b_{1,\mathcal{T}}(z),\ldots ,b_{K,\mathcal{T}}(z)\}^{\top}$. All the spline basis functions $b_{k,\mathcal{T}}(z)$ are nonnegative; therefore, they are all valid probability density functions.  For any polynomial spline $s(z)$, it can be uniquely represented via a linear combination of spline basis functions, that is, $s(z) = \sum_{k=1}^K \gamma_kb_{k,\mathcal{T}}(z)$.

Define $\mathcal{H}^{(\varrho)}(\mathcal{T})$ as the space of functions $\psi$ on $\mathcal{T}$ whose $\nu$-th derivative exists and satisfies a Lipschitz condition of order $\delta:\left|\psi^{(\nu)}(z)-\psi^{(\nu)}\left(z^{\prime}\right)\right| \leq C_{v}\left|z-z^{\prime}\right|^{\delta}$, for $z, z^{\prime} \in \mathcal{T}$ and $\varrho = \nu + \delta$. The following Lemma \ref{LEM:approxi} can quantify the approximation power of polynomial splines for functions within $\mathcal{H}^{(\varrho)}(\mathcal{T})$, suggesting smooth functions can be well approximated when the knot number increases to infinity. 
\begin{lemma}[\cite{schumaker2007spline}]
\label{LEM:approxi}
For any function  $\psi \in \mathcal{H}^{(\varrho)}(\mathcal{T})$, there exists a spline $\psi^{\ast} \in \mathcal{U}$, such that
$\sup_{z \in \mathcal{T}}|\psi^{\ast}(z)-\psi(z)|\leq C H^{-(\varrho+1)}$ for some positive constant $C$.
\end{lemma}

\begin{remark}
\label{Remark1}
According to Lemma \ref{LEM:approxi}, for density function $p(z) \in \mathcal{H}^{(\varrho)}(\mathcal{T}^o)$, where $\mathcal{T}^o \subseteq \mathbb{R}$ could be either finite or infinite support, there exists $\widetilde{p}(z) \in \mathcal{H}^{(\varrho)}(\mathcal{T})$  such that is $\widetilde{p}(z)$ a valid density function and $\sup_{z \in \mathcal{T}^o}|\widetilde{p}(z)-p(z)|\leq C H^{-(\varrho+1)}$.
\end{remark}

\section{\uppercase{Nonparametric Posterior Approximation with Spline}}

Remark \ref{Remark1} emphasizes the capability of spline functions to approximate complex distributions provided sufficient interior knots. To enhance the approximation of posteriors with arbitrary shapes, we introduce Spline Automatic Differentiation Variational Inference (S-ADVI). This nonparametric approach aims to represent posteriors as spline functions, allowing for more flexible and accurate modeling.

\subsection{Spline Automatic Differentiation Variational Inference (S-ADVI)}
Following the framework of nonparametric Bayesian inference, we assume the true posterior $p({ z|\bm x})$ is within an infinite dimensional space such that $p({z|\bm x}) \in \mathcal{H}^{(\varrho)}(\mathcal{T}^o)$. Therefore, according to Lemma \ref{LEM:approxi} and Remark \ref{Remark1}, the posterior distribution of a latent variable $z$, $p(z| \bm x)$ can be well approximated by a spline function: $q_\phi(z|\bm x) = \sum_{k=1}^{K} \gamma_{k}(\bm x) b_{k,\mathcal{T}}(z) = \bm b (z)^{\top} \bs{\gamma}(\bm x)$. See Figure \ref{fig:spline} to illustrate the density functions based on the linear combinations of spline basis functions. By the definition of normalized spline basis in Section \ref{SEC:Spline-App}, 
$\int_{\mathcal{T}} b_{k,\mathcal{T}}(z) dz =1$, for $k=1, \ldots, K$. Therefore, to ensure that $\int_\mathcal{T} q_\phi(z|\bm x) dz = 1$ and $q_\phi(z|\bm x) > 0$ for $z \in \mathcal{T}$, the spline coefficients must satisfy $\gamma_{k}(\bm x) \geq 0$ and $\sum_{k=1}^K\gamma_{k}(\bm x) = 1$. For notation simplicity, we denote $b_k(z) = b_{k, [0,1]} (z)$ for the rest of paper. We consider the mean-field assumption and assume the latent variables to be independent of each other.  For $j$-th latent variable, we use $\mu_j(\bm x)$ and $\sigma_j(\bm x)$ for location-scale transformations for latent variable $z_j$, such that $z_j = \mu_j(\bm x) + \sigma_j(\bm x) \epsilon_j$, and $\epsilon_j \in [0,\,1]$ is a random variable. The location-scale transformation allows for adaptive supports of posteriors on $[\mu_j(\bm x), \mu_j(\bm x) + \sigma_j(\bm x)]$, where $\mu_j(\bm x)$ and $\sigma_j(\bm x)$ are unknown parameters to be estimated. The proposed posterior is determined by unknown parameters $\mu_j(\bm x)$, $\sigma_j(\bm x)$, and the spline coefficients $\{\gamma_{jk}(\bm x), k = 1, \ldots, K\}$, which capture the location, scale, and shape of the distribution. With pre-specified spline basis functions in the S-ADVI, we can use spline coefficients to represent the shape of approximated posteriors. {Section \ref{SEC:Single-Column} demonstrates using spline coefficients to investigate the relationship between the shape of posteriors and input features, enhancing the model interpretation.} Let $J$ be the total number of latent variables. Collectively, for the vector of latent variables $\bm z = \{z_1, \ldots, z_J\}$, the posterior $p\left(\bm z| \bm x\right)$ can be approximated by
\begin{align}
\label{EQU:S-ADVI}
   q_\phi\left(\bm z| \bm x\right) =  \prod_{j=1}^J\sigma^{-1}_j(\bm x) \cdot q_\phi\left(\bm \epsilon| \bm x\right) = 
   \prod_{j=1}^J \frac{1}{\sigma_j(\bm x)} \cdot \sum_{k=1}^K \gamma_{jk}(\bm x) b_k\left(\frac{z_j-\mu_j(\bm x)}{\sigma_j(\bm x)}\right),
\end{align}
where $\bm \epsilon = \{\epsilon_1, \ldots, \epsilon_J\}$ is the vector of latent variables before the location-scale transformation. In the S-ADVI, the parameters of approximation family are $\phi = \{\mu_j(\bm x), \sigma_j(\bm x), \gamma_{jk}(\bm x), k = 1, \ldots, K, j=1, \ldots, J\}$.
The objective of the proposed S-ADVI is to maximize the IWAE defined in (\ref{EQU:IWAE}), where the term $q_\phi\left(\bm z_t| \bm x\right)$ is given in (\ref{EQU:S-ADVI}). The spline degree $\varrho$, and the number and values of interior knots are hyperparameters to be specified. In the numerical studies, we choose the cubic spline ($\varrho = 3$) with equal-space knots, which is commonly used in nonparametric model estimation \citep{hastie2017generalized,Yu:etal:20}. The influence of the number of interior knots is evaluated in the experiments in Section \ref{SEC:Single-Column}.

\begin{remark}
It is possible to get correlated 
$\bm z$ by taking a linear transformation 
$\bm z = \bm \mu + \mathbf{\Sigma} \bm \epsilon$, where $\mathbf{\Sigma}$ is an unknown covariance matrix and the random variables in $\bm \epsilon$ are independent of each other. For complex, nonlinear dependency between components, multivariate posterior estimation is viable through multivariate spline approximation to capture relationships between latent variables, which is considered as a  future study. 

\end{remark}

\subsection{Model Estimation}
\label{SEC:modelestimation}
The spline posterior approximation can be regarded as a mixture of density functions based on spline bases. The stratified ELBO (SELBO)/IWAE (SIWAE) \citep{roeder2017sticking, morningstar2021automatic} are common methods to optimize ELBO/IWAE with mixture densities. However, in S-ADVI, directly applying the stratified techniques is challenging, which involves the summation of products of all the combinations of spline coefficients and latent variables. When the number of latent variables is large, the stratified methods can be computationally expensive.

We tackle the above-mentioned challenge via the concrete distribution \citep{maddison2017concrete}, which can be used as an approximation to the categorical distribution. The concrete distribution has two parameters, $\bm \alpha \in (\mathbb{R}^{+})^K$ and $\Lambda$, where $\sum_k \alpha_k=1$. When $\Lambda \rightarrow 0$, the concrete distribution approaches the categorical distribution with the event probability vector being $\bm \alpha$. However, if $\Lambda$ is fixed low, the concrete approximation cannot explore different combinations of spline bases, leading to poor model estimation. We use an annealing approach \citep{abid2019concrete}, where we start model training with a high $\Lambda=\Lambda_0$, and gradually reduce $\Lambda$ after each epoch. In this paper, we use $\Lambda(c)=\Lambda_1+ (\Lambda_0-\Lambda_1)e^{-c/\eta}$ to smoothly reduce the temperature, where $\Lambda_1$ is the final temperature, $c$ denotes the current epoch number, and $\eta$ controls the speed of decay. As shown in the experiments (Section \ref{app_sec:annealing_exp} in supplementary material), S-ADVI is not sensitive to the choice of annealing functions.

To this end, we summarize the Stochastic Backpropagation \citep{rezende2014stochastic} for estimating S-ADVI.
\noindent \textbf{Generating random samples from mixture models.} One key component of the proposed S-ADVI method is to generate random samples from a mixture model with distribution $q_\phi\left(z_j | \bm x\right) = \sum_{k=1}^K \gamma_{jk}(\bm x) \widetilde{b}_{k,\mathcal{T}}(z_j; \bm x)$, where $\widetilde{b}_{k,\mathcal{T}}(z_j; \bm x) = \sigma_j(\bm x)^{-1} b_k\left\{\sigma_j^{-1}(\bm x)[z_j-\mu_j(\bm x)]\right\}$. A hierarchical approach to generate random samples from mixture models involves two steps: generating random samples for distributions $\widetilde{b}_{k,\mathcal{T}}(z_j; \bm x)$ and randomly selecting one sample with probability $\gamma_{jk}(\bm x)$ for each $z_j$. However, it is not straightforward to sample from $\widetilde{b}_{k,\mathcal{T}}(z_j; \bm x)$ and apply the reparameterization trick to the categorical distribution. Utilizing the pre-specified spline bases, with concrete approximation, at each iteration, we consider the following procedures:

\begin{enumerate}[noitemsep,topsep=0pt]
\setlength\itemsep{0.1em}
    \item Use the Metropolis-Hastings algorithm to generate a sequence of random samples from the distribution $b_{k}(\epsilon_j)$ and then randomly pick $w_{jk}$ from generated samples.
    \item Generate random sample $\bm u_j$ from a concrete distribution with $\Lambda=\Lambda(c)$ and $\alpha_{jk}=\gamma_{jk}(\bm x)$. 
    \item Define $\epsilon_j = \sum_{k=1}^{K} u_{jk}w_{jk}.$
    The property of concrete distribution guarantees that when $\Lambda(c) \rightarrow \Lambda_1$ as $c$ increases, the procedure well approximates the discrete hierarchical sampling process.
\end{enumerate}

\noindent \textbf{Backpropagation with reparameterization trick.} We aim to differentiate the objective function w.r.t. the parameters $\phi$ via a Monte Carlo approximation.   The Monte Carlo approximation is based on the random samples generated from the mixture of spline basis functions. To obtain the differentiation, we consider the following reparameterization trick for our objective function $\mathcal{L}_{\text{IWAE}}\{\phi(\bm x)\}$:
\begin{align}
	     \label{EQU:backpropagation}   
      \mathbb{E}_{\left\{\bm \epsilon_t\right\}_{t=1}^T}\left[\log \frac{1}{T} \sum_{t=1}^T \frac{p_{\theta}\left\{\bm x,\bm \mu(\bm x) + \bm \sigma (\bm x) \cdot \bm \epsilon_t \right\}}{\prod_{j=1}^J\left\{\sum_{k=1}^K \gamma_{jk}(\bm x) b_k\left(\epsilon_{jt}\right) \right\}}\right] +  \sum_{j=1}^J\log \sigma_j(\bm x),
\end{align}
where $\bm \epsilon_t = \{\epsilon_{1t}, \ldots, \epsilon_{Jt}\}^{\top}$ are variables generated from $\sum_{k=1}^{K} \gamma_{jk}(\bm x) b_{k}(\epsilon_{jt})$,  $\bm \mu(\bm x) = \{\mu_j(\bm x), j = 1, \ldots, J\}$ and $\bm \sigma(\bm x) = \{\sigma_j(\bm x), j = 1, \ldots, J\}$ are vectors of location and scale parameters. The $\log \sigma_j(\bm x), j =1, \ldots, J$ terms in (\ref{EQU:backpropagation}) prevents the model from degenerating into a deterministic model. Derivative of (\ref{EQU:backpropagation}) can be found in Section \ref{apdx:reparam_trick} of the supplementary material.

\noindent \textbf{Penalized spline.} In nonparametric smoothing, penalized spline captures intricate data patterns with regularization (roughness penalty) to prevent overfitting and manage the complexity of the fitted function \citep{Wood:03}. 
Here, we consider a roughness penalty for a spline function $s(\cdot)$, defined as $\mathcal{E}(s) = \int_{\mathcal{T}} \{s''(t)\}^2dt$, to control the complexity of the fitted curve and avoid overfitting. According to properties of spline polynomials, $\mathcal{E}(s) = \bm \gamma^{\top} \mathbf{P} \bm \gamma$, where the matrix $\mathbf{P}$ is a $K$ by $K$ positive definite matrix, see Section \ref{SEC:Implementation} of the supplementary material for the detailed definition.  Then, the objective function becomes $\mathcal{L}^P_{\text{IWAE}}\{\phi(\bm x)\} = \mathcal{L}_{\text{IWAE}}\{\phi(\bm x)\}+ \lambda \bm \gamma^{\top} \mathbf{P}\bm \gamma.$

\section{\uppercase{Properties of S-ADVI}}

While existing works have demonstrated that the spline approximation has an upper bound on the approximation error for any function within the functional space $\mathcal{H}^{(\varrho)}(\mathcal{T}^o)$ (Lemma \ref{LEM:approxi}), it is worth examining the approximation power of the proposed S-ADVI based on spline approximation. In this section, we start with the bound of IWAE (Theorem \ref{THE:ELBO}), which implies the bound of KL divergence of the spline density approximation from the true posterior (Theorem \ref{THE:Approximation}). Then, Theorem \ref{THE:Approximation-Regression} quantifies the posterior approximation error between the S-ADVI estimator and true posterior.

We first state the necessary assumptions to facilitate our theoretical studies. 
\begin{itemize}[noitemsep,topsep=0pt]
    \item[\textit{(A1)}] The prior $p(\bm z)$ and the likelihood function $p_{\theta}(\bm x| \bm z)$ are bounded over the support regions. 
    \item[\textit{(A2)}] The true posteriors can be fully factorized, that is, $p(\bm z| \bm x) = \prod_{j=1}^J p(z_j | \bm x)$. For $j=1, \ldots, J$, the posterior $p(z_j|\bm x) \in \mathcal{H}^{(\varrho)}(\mathcal{T}^o)$. There exists some region $\mathcal{T}^{\ast} \subset \mathcal{T}^o$ such that $\int_{\mathcal{T}^o - \mathcal{T}^{\ast}} p(z_j | \bm x) dz_j < \epsilon$ for some $\epsilon >0$. In addition, the posterior $p(z_j| \bm x)$ are bounded by some constant over $\mathcal{T}^{\ast}$. 
    \item[\textit{(A3)}] There exist constants $C_1$ and $C_2$ such that $C_1 H^{-1} \leq \upsilon_{h} - \upsilon_{h-1} \leq C_2 H^{-1}$ for $1\leq h \leq H$. 
\end{itemize}
\vspace{1ex}

\begin{remark}
    Assumption (A1) is a mild assumption on the prior $p(\bm z)$ and the likelihood function, which can be easily satisfied.  Assumptions (A2) -- (A3) are typical assumptions under the framework of spline approximation \citep{wang2009spline,Yu:etal:20}. Assumption (A3) assumes that the true posteriors can be fully factorized. For more general cases, the true posterior can be factorized into $\prod_{m=1}^M p(\bm z_{s_m}|\bm x)$, where $s_m$ is an index set of the latent variables and $\cup_{m=1}^M s_m = \{1, \ldots, J\}$. Applying the functional ANOVA results in \citep{stone1994use}, we can derive the $L_2$ approximate rate is $H^{-(\varrho^{\ast} + 1)}$, where the $\varrho^{\ast}$ is a suitably lower bound to the smoothness of the components $p(\bm z_{s_m}|\bm x), m =1, \ldots, M$. The above approximation result is a general form of multivariate posterior approximation, allowing complex interactions between latent variables. 
\end{remark}

\setcounter{theorem}{0}

Lemma \ref{THE:ELBO} shows that the proposed S-ADVI allows us to quantify the lower bound of IWAE.
 \begin{lemma}
 \label{THE:ELBO}
  Under Assumptions (A1) -- (A3), the optimal IWAE is bounded by $\log p_{\theta}(\bm x) - C J H^{-(\varrho+1)} - J \epsilon$, where $C$ is some positive constant.
\end{lemma}
Theorem \ref{THE:Approximation} quantifies the variational approximation error with respect to the class defined in (\ref{EQU:S-ADVI}). See Section \ref{SEC:THE:Approximation-Proof} in the supplementary material for the detailed proof.
 
 \begin{theorem}
 \label{THE:Approximation} Under Assumptions (A1) -- (A3), the difference between the true posterior and the spline estimator is bounded by the order of $H^{-(\varrho+1)}$, that is, there exists a constant $C$, such that $D_{KL}\{q_{\phi(\bm x)}(\bm z)||p(\bm z| \bm x)\} \leq CJ(H^{-(\varrho+1)}+\epsilon)$.
\end{theorem}

We consider the posterior approximation based on the observed data $\{\bm x_1, \bm x_2, \ldots, \bm x_n\}$. When the observed data points $\bm x_i$ and $\bm x_{i'}$ are close enough, the corresponding posteriors $p(\bm z|\bm x_i)$ and $p(\bm z|\bm x_{i'})$ are close. For any given posterior $p(\bm z|\boldsymbol x)$ satisfying Assumption (A2), there exists a density function based on spline approximation $\prod_{j=1}^J \sum_{k=1}^K \gamma^\ast_{jk}(\boldsymbol x) b_{k, \mathcal{T}^\ast_j}(z_j)$ with $\mathcal{T}^\ast_j = [\mu^\ast_j(\boldsymbol x), \mu^\ast_j(\boldsymbol x) + \sigma^\ast_j(\boldsymbol x)]$ close to $p(\bm z|\boldsymbol x)$ with differences bounded by $JH^{-(\varrho+1)}$. Under some mild assumptions, $\gamma^\ast_{jk}(\boldsymbol x)$, $k=1,\ldots, K$, $\mu^\ast_j(\boldsymbol x)$, and $\sigma^\ast_j(\boldsymbol x)$ can be well approximated by nonparametric regression, such as the deep neural network. Specifically, the objective function can be formulated as $\sum_{i=1}^n \mathcal{L}_{\text{IWAE}}\{\phi(\bm x_i)\}$ and $\widehat{\phi}(\bm x) = \argmax_{\phi}\sum_{i=1}^n \mathcal{L}_{\text{IWAE}}\{\phi(\bm x_i)\}$, where $\phi(\bm x_i)$ is the collection of parameters of the S-ADVI estimators.

Theorem \ref{THE:Approximation-Regression} quantifies posterior approximation error generated from the nonparametric smoothing and theoretical differences between the S-ADVI estimator and true posterior. See Section \ref{SEC:THE:Approximation-Regression-Proof} in the supplementary material for the detailed proof.
 
 \begin{theorem}
 \label{THE:Approximation-Regression} Under Assumptions (A1) -- (A3), the average KL divergence of the spline estimator from the true posterior satisfies has 
\begin{align*}
\lim_{n \to \infty} \textrm{Pr}\Big[n^{-1} \textstyle \sum_{i=1}^n D_{KL}\{q_{\widehat{\phi}(\bm x_i)}(\bm z)||p(\bm z| \bm x_i)\}\leq CJ (H^{-2(\varrho +1)} + \epsilon^2 +  H^2  \Delta^2)\Big] = 1,
\end{align*}
 where $C$ is a positive constant and $\Delta$ is the $L_2$ estimation error of nonparametric regression for $\mu_j(\bm x)$, $\sigma_j(\bm x)$, and $\gamma_{jk}(\bm x), k = 1, \ldots, K, j=1, \ldots,J$.
\end{theorem}

\begin{remark}
    Theorem \ref{THE:Approximation} suggests increasing the number of interior knots $H$ can reduce the S-ADVI approximation errors. However, when applied to real data analysis, according to the results in Theorem \ref{THE:Approximation-Regression}, choosing the optimal number of interior knots balances approximation bias and estimation variance. In addition, the roughness penalty in penalized spline is used to avoid overfitting.
\end{remark}

\begin{remark} (\textit{Example of the convergence rate $\Delta$}) We assume that all the components are compositions of several functions. Suppose that $\phi(\boldsymbol x) = g_{q} \circ g_{q-1} \circ \cdots \circ g_{1} \circ g_{0},$ where $g_{i}:[a_{i}, b_{i}]^{d_{i}} \rightarrow [a_{i+1}, b_{i+1}]^{d_{i+1}}$. Denote by ${g_{ij}, j=1, \ldots, d_{i+1}}$ the components of $g_{i}$. Let $t_{i}$ be the maximal number of variables on which each of the $g_{i j}$ depends, and each $g_{i j}$ is a $t_{i}$-variate function. Then, the convergence rate is $\Delta = \max_{i=0, ..., q} n^{-(2 \beta_{i}) /(2 \beta_{i}+t_{i})},$ where $\beta_{i}, i=0, \ldots, q$ are degrees of Hölder smoothness conditions of functions $g_{ij}$ \citep{NonparametricRegression}.
\end{remark}

\begin{figure*}[!htbp]
\begin{center}
    \includegraphics[width=0.9\textwidth]{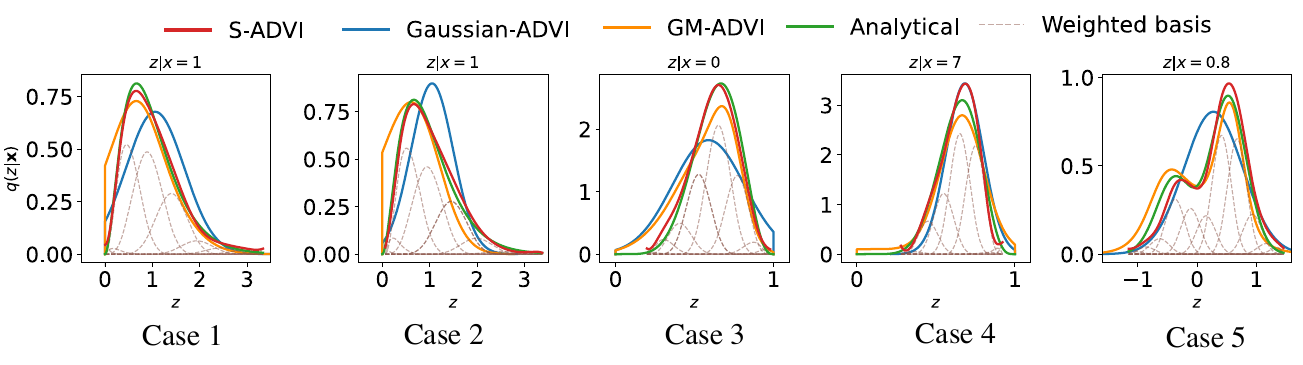}
    \vspace{-1ex}
\caption{\capitalisewords{Posterior approximation results with S-ADVI, Gaussian-ADVI, and GM-ADVI for Cases 1--5}}
\label{fig:toy_examples_1}
\end{center}
\end{figure*}

\section{\uppercase{Related Works}}

Variational inference has traditionally relied on parametric approximations, but efforts to enhance flexibility have led to various approaches. Gaussian mixture approximation offers a flexible posterior approximation but is susceptible to issues like posterior collapse \citep{Nonparametric2012}. Normalizing flow, employing invertible functions to model posterior distributions, provides an alternative flexible approach \citep{rezende2015variational, rezende2020normalizing, wu2020stochastic}. Neural Spline Flows utilizes monotonic and element-wise rational-quadratic spline as building blocks of normalizing flow \citep{durkan2019neural}. Boosting variational inference, a mixture-based approximation, offers flexibility but poses implementation and interpretation challenges \citep{locatello2018boosting, kobyzev2020normalizing}. On the other hand, the estimation of the density of the non-parametric kernel resembles a Gaussian mixture with numerous components \citep{Nonparametric2012}.

Implicit processes represent another avenue for facilitating flexible inferences utilized across Bayesian neural networks, neural samplers, and data generation frameworks. The enhancement of priors and posteriors through approximate inference techniques is well-documented \citep{ma2019variational, ortega2022deep, takahashi2019variational, shi2017kernel, molchanov2019doubly}. A notable method for training implicit models involves the nonparametric approximation of log density, known as the score estimator, which has been explored in recent studies \citep{KEF, Stein, shi2018spectral}. Furthermore, a comprehensive examination and convergence analysis of existing score estimators have been presented, offering a unified perspective on this methodology \citep{zhou2020nonparametric}. Despite these advancements, estimating implicit posteriors, especially in models characterized by high-dimensional latent variables, remains a significant challenge \citep{rodriguezsantana22a}.

Research on variational approximations has explored theoretical guarantees, convergence, optimization techniques, and model-specific analyses. Frequentist consistency has been considered \citep{wang2019frequentist, ZhangGao2020}, as well as overparameterized Bayesian Neural Networks \citep{huix2022variational}. Notably, existing theoretical studies have primarily focused on parametric distribution families.

\section{\uppercase{Results}}
In this section, we demonstrate the proposed method with experiments on both simulated and real datasets. All experiments are based on PyTorch 2.0 \citep{paszke2019pytorch} running on a Nvidia A100 80G GPU \footnote[1]{Example codes are available at: \url{https://github.com/TianshuFeng/SADVI_AISTATS2024}}. 

\subsection{Posterior Approximation}

\begin{figure*}[!tb] 
    \centering
    \includegraphics[width=0.85\textwidth]{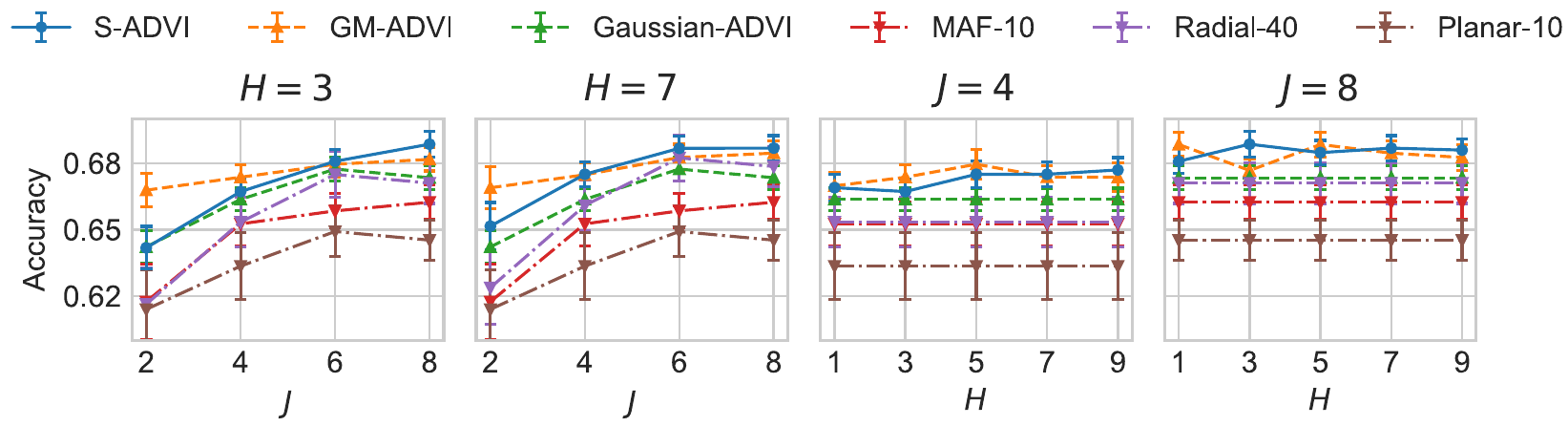}
    \vskip -0.1in
	\caption{\capitalisewords{Performance comparison for single column classification for the FMNIST dataset with a range of latent variables $J$ and interior knots $H$ (Error bars denote the standard error)}} \label{fig:performace}
\end{figure*}

{We consider the following five simulation cases to demonstrate the proposed method (S-ADVI) in approximating the posterior distribution of $z|x$:}
\begin{enumerate}[noitemsep,topsep=0pt]
    \item $z\!\sim\!\textrm{Gamma}(2,2),\, x|z\!\sim\! \textrm{Exponential}(z)$;
    \item $z\!\sim\!\textrm{Gamma}(2,2),\, x|z\!\sim\!\textrm{Poisson}(z)$;
    \item $z\!\sim\!\textrm{Beta}(7,3),\, x|z \!\sim\!\textrm{Bernoulli}(z)$;
    \item $z\!\sim\!\textrm{Beta}(2,2),\, x|z\!\sim\!\textrm{Binomial}(10,z)$;
    \item $z\!\sim\! 0.5 N(-0.5, 0.1) + 0.5N(0.5, 0.1)$,  $x|z\!\sim\! N(z, 1)$.
\end{enumerate}

\begin{table*}
    \centering
    \caption{\capitalisewords{Mean and Standard Deviation (in bracket) of Root Integrated Squared Error in Posterior Approximation}}\label{tab:my_label}
    \vspace{1ex}
    {
    \addtolength{\tabcolsep}{-0.4em}
      \begin{tabular}{lccccc}
    \hline\hline
        Method & Case 1         & Case 2        & Case 3        & Case 4        & Case 5\\ 
        \hline
       Gaussian-ADVI    & 0.408 (0.274)          & 0.239 (0.044)         & 0.630 (0.239)        & 0.631 (0.386)         & 0.243 (0.014) \\ 
    GM-ADVI    & 0.353 (0.152)          & 0.211 (0.026)         & 0.403 (0.115)        & 0.395 (0.107)         & 0.137 (0.020) \\
     S-ADVI    & \textbf{0.086} (0.046) &\textbf{0.054} (0.020) &\textbf{0.211} (0.070) &\textbf{0.310} (0.080) &\textbf{0.097} (0.025) \\
       \hline
       \bottomrule
    \end{tabular}
    }
    \vspace{-2ex}
\end{table*}

We generate 1024 samples for all cases, training models in batches of 32 across 40 epochs over 20 runs. A two-layer multilayer perceptron (MLP) with 20 hidden units per layer is used to estimate unknown parameters. For S-ADVI, we set the interior knots ($H$) to 6 for Cases 1-4, and 9 for the final case due to its complex multimodal structure. Performance is evaluated using root integrated squared error (RISE), defined as $[\int {[q(z|x)-p(z|x)]}^2dz]^{1/2}$, comparing our method against Gaussian-ADVI (approximating $p(z|x)$ with a Gaussian distribution) and Gaussian Mixture ADVI (GM-ADVI, \cite{morningstar2021automatic}) based on stratified sampling. For likelihoods requiring bounded support of $z|x$, truncated distributions of $z|x$ are considered for Gaussian-ADVI and GM-ADVI.

The results presented in Table \ref{tab:my_label} underscore the advantage of S-ADVI over Gaussian-ADVI and GM-ADVI, and Table \ref{tab:complete_posterior_infer_res} in the supplementary materials shows additional comparison results with methods based on normalizing flows. Visual representations in Figure \ref{fig:toy_examples_1} and Figure \ref{fig:all_cases} (supplementary material) depict the approximated posteriors compared to the true posterior functions, showcasing S-ADVI's ability to approximate the true posterior distribution. Across Cases 1 to 5, S-ADVI consistently outperforms other methods by leveraging spline functions for posterior approximation.
Specifically, for Cases 1 through 3 and Case 5, where the true posterior distribution exhibits significant skewness or complex multimodal nature, Gaussian-ADVI effectively estimates the posterior locations but struggles to capture their shapes.
In Cases 1 through 3, characterized by strong skewness and bounded supports in the true posterior distribution, GM-ADVI exhibits subpar performance. While GM-ADVI captures the skewness of the posteriors, it fails to estimate the boundaries, showing a high estimation error when $z$ is close to the extreme points.
In contrast, our proposed methodology excels, particularly when dealing with distributions that are skewed and with bounded latent variable supports. Additionally, 
Figure \ref{fig:all_cases} shows that normalizing flows effectively capture the posterior distribution's shape. In complex scenarios like the multimodal distributions of Case 5, normalizing flows outperform the GM-ADVI and Gaussian-ADVI through invertible density transformations. However, normalizing flows generates unexpected irregularities, such as wiggles in Cases 1 and 2. One possible reason is that, in contrast to the complex invertible density transformations of normalizing flows, S-ADVI provides a more effective way to approximate posterior distributions.

\subsection{Real Data Applications}
\label{SEC:Single-Column}

\begin{figure}[!tb] 
    \centering
    \includegraphics[width=0.4\textwidth]{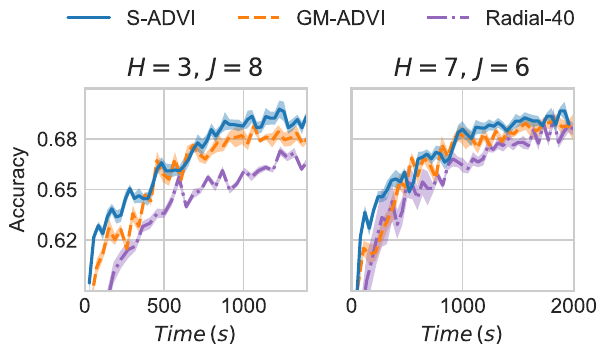}
    \vskip -0.1in
	\caption{\capitalisewords{Computational budget comparison for single column classification for the FMNIST dataset with a range of latent variables $J$ and interior knots} $H$ \capitalisewords{(Shadows represent the standard error)}} \label{fig:computational}
 \vspace{-2ex}
\end{figure}

\begin{figure*}[!tbp] 
	\centering
	\subfloat{\includegraphics[width=0.92\textwidth]{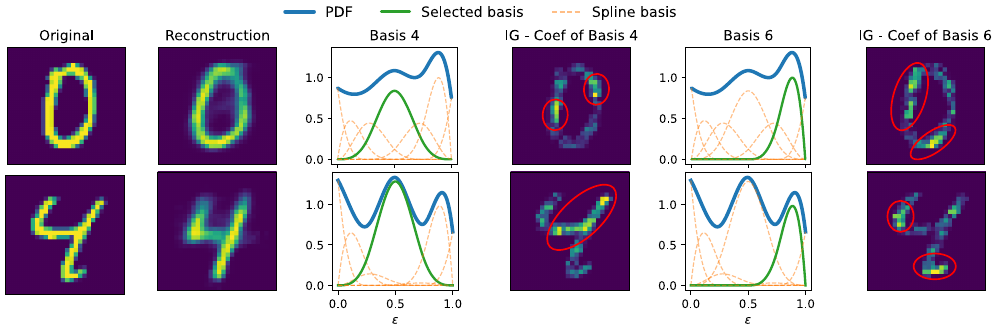}}
	\\ \vskip -0.1in
	\caption{\capitalisewords{Analysis of the approximated shape of posterior from S-VAE, where Brighter pixels correspond to higher absolute attributes, and Regions containing highly attributed pixels are highlighted in red circles}} \label{fig:latent_individual_ig}
\end{figure*}

\noindent \textbf{Single Column Classification.} We conduct classification tasks on the 
Fashion-MNIST (FMNIST) \citep{xiao2017online} datasets to evaluate the performance of S-ADVI under different parameters. The model used for classification is Variational Information Bottleneck (VIB) \citep{alemi2017deep}, which shares a similar structure as VAE with decoder replaced by a classifier. We use a simple multinomial logistic regression as the classifier and the isotropic Gaussian distribution prior to encouraging the encoder to capture the underlying information present in the images. To showcase the flexible spline approximation, we limit the input samples only to include the center column of training images. An illustration example of the input sample can be found in Figure \ref{fig:singlecolumn} in the supplementary material.

We assess S-ADVI's performance with $H = {1, 3, 5, 7, 9}$, $J = {2, 4, 6, 8}$ and $T = 10$ in IWAE, setting $v_0 = 0$ and $v_{H+1} = 1$ with equal spacing for $v_h$ where $h=1,\ldots, H$. For comparison, we consider VIB with GM-ADVI and Gaussian-ADVI, matching the number of Gaussian components in GM-ADVI to the spline bases in S-ADVI ($H+4$). Three normalizing flow-based methods, Planar and Radial (10 and 40 flows), and masked autoregressive flow with 10 flows (MAF-10) are also compared. Prediction scores are calculated by averaging model outputs from 100 samples drawn from $q_{\phi}(\bm z| \bm x)$. Experiments are repeated 10 times, and we reporting the mean and standard error of accuracy on the test set.

Figure \ref{fig:performace} shows the performance comparison on the FMNIST dataset for the top six methods across varying $J$ and $H$ for S-ADVI ($H+4$ basis for GM-ADVI).
Results indicate S-ADVI's performance generally improves with increased latent variables and interior knots, aligning with theoretical outcomes. S-ADVI benefits from additional interior knots when the number of latent variables is relatively small. With sufficient latent variables, extra interior knots is less impactful and can lead to overfitting. GM-ADVI's flexible Gaussian components provide an edge over S-ADVI with limited $J$ and $H$. However, adding interior knots or latent variables can help S-ADVI to match and outperform GM-ADVI, which may result from 
the natural regularization of the latent variables with bounded support of spline approximation.
Both S-ADVI and GM-ADVI outperform the normalizing flow methods, especially with limited latent variables.

In addition, we compare the computational budget of the top 3 performing methods: S-ADVI, GM-ADVI, and Radial-40. Results are shown in Figure \ref{fig:computational}. In general, S-ADVI converges faster than GM-ADVI and Radial-40, especially when the numbers of knots and latent variables are limited. One of the potential reasons could be the fewer parameters required by S-ADVI. When the number of latent variables increases, three methods achieve comparable convergence speeds. Annealing is important in the proposed method to ensure most of the combinations of coefficients are explored. Section \ref{app_sec:annealing_exp} of the supplementary file evaluates the influences of annealing functions on model performance. The performance of our proposed method is robust to the choice of annealing function.

\noindent\textbf{Imaging Reconstructions.} Our previous experiments suggest that the proposed S-ADVI methods can match and outperform other VI-based methods in terms of classification task, even when the data does not contain sufficient information. In this experiment, we further compare the S-ADVI and GM-ADVI methods in terms of reconstruction using MNIST \citep{lecun1998gradient}, FMNIST, and CIFAR-10 \citep{krizhevsky2009learning} datasets. Section \ref{app_sec:full_image_exp} in the supplementary material presents the implementation details and the numerical results. We find that when the number of latent variables is limited, GM-ADVI outperforms S-ADVI. When $J$ increases, the performance of both models improves, and S-ADVI outperforms GM-ADVI. The number of bases $H$ also influences model performance when $H$ increases from 1 to 3, but further increasing $H$ does not improve model performance much and can lead to overfitting. It is also interesting to note that the advantage of S-ADVI increases for more complicated datasets (e.g., MNIST vs CIFAR-10) with higher $J$ and $H$.

\noindent \textbf{Interpreting Distributions of Latent Variables.} We evaluate the relationship between spline approximations' shape and input features by training a variational autoencoder with S-ADVI (S-VAE) on the MNIST benchmark dataset to reconstruct input images. Implementation details and overall results are provided in Section \ref{SEC:shape-supp} (supplementary material). We assess the shape of posterior distribution pre location-scale transformation, focusing on the density functions of $\epsilon_j|\bm x$ defined in (\ref{EQU:S-ADVI}), with support $q_{\phi}(\epsilon_j|\bm x)$ in $[0,1]$ for all $j=1,\ldots,J$, so that the shapes are comparable across samples.

Using Integrated Gradients (IG) \citep{sundararajan2017axiomatic}, we explore the link between input samples and approximated $q_{\phi}(\epsilon_j | \bm x)$ shapes. For a sample $\bm x$ and its density function $\sum_{k=1}^K \gamma_{jk}(\bm x) b_k\left(\epsilon_j\right)$, IG attributes $\gamma_{jk}(\bm x)$ values to input sample features, here MNIST pixels. We measure feature relative importance to $\gamma_{jk}(\bm x)$ using IG's absolute attribute values. We set the baseline to a zero vector, representing a blank image.

Experiment results in Figure \ref{fig:latent_individual_ig} show original samples and S-VAE reconstructions (Columns 1 and 2). We selected bases 4 and 6 of $q_{\phi}(\epsilon_4|\bm x)$ with distinct modes for different digits (Figure \ref{fig:latent-pdf}). Columns 3 and 5 illustrate selected and other weighted spline bases, while Columns 4 and 6 display IG results, with brighter pixels indicating higher attribute values. We notice that bases with higher coefficients represent larger image regions and capture key image characteristics. For instance, IG reveals that in the third image, basis 4 corresponds to digit 4's upper right structure, and basis 6 to the top left part of digit 0.

\section{\uppercase{Discussion}}\label{sec:discussion}

In this paper, we introduce a nonparametric ADVI framework that uses spline approximations to approximate posterior distributions and achieves a balance between flexibility, parsimony, and interpretability. We establish S-ADVI’s posterior consistency in approximating complex distributions through the asymptotic properties of IWAE. Compared with classic ADVI methods, experiment results suggest S-ADVI's superior capacity to approximate distributions with bounded support and multimodality. 

Despite its strengths, the proposed S-ADVI can be further improved. First, our methodology uses the mean-field approximation for model simplicity and computational efficiency, but it does not account for the underlying dependencies among latent variables.  Second, we can further optimize the locations of interior knots, which are pre-specified in the current framework, based on dataset characteristics \citep{spiriti2013knot}. In addition, incorporating techniques such as fused lasso could help in subgroup analysis on spline coefficients \citep{tibshirani2005sparsity}. For the posterior inference problems in VI, latent variables can be divided into two categories: local latent variables for individual observations and global latent variables \citep{wang2019frequentist}.  This paper focuses primarily on the posterior inference of local latent variables, as discussed in the theoretical analysis and numerical evaluations.  Investigating the theoretical properties of global latent variables represents promising future research. Finally, our spline-based posterior approximation approach opens up possibilities for modeling spatial-temporal data within the ADVI framework, which aligns with the current research trends towards more accurate prior approximations in variational autoencoders \citep{pang2020learning}. 

\section*{Acknowledgements}

We would like to thank the reviewers for valuable feedback and suggestions.
\newpage
\part{Supplementary Materials for ``Nonparametric Automatic Differentiation Variational Inference with Spline Approximation''}
\counterwithin{figure}{section}
\renewcommand\thefigure{S.\arabic{figure}}   
\renewcommand\thesection{S.\arabic{section}} 
\renewcommand\thetable{S.\arabic{table}}
\section{\uppercase{Additional Lemmas and Proof Details}}
\label{SEC:spline}

\subsection{Proof of Theorem \ref{THE:Approximation}}
\label{SEC:THE:Approximation-Proof}
Lemma \ref{THE:ELBO} shows that the proposed S-ADVI allows us to quantify the Lower Bound of IWAE.

\begin{proof} [Proof of Lemma \ref{THE:ELBO}] We start with a one-dimensional latent variable with a finite support $\mathcal{T}$. The objective function is 
$\mathcal{L}_{\text{IWAE}}(\phi) = \mathbb{E}_{\left\{z_t \sim q_\phi(z| \bm x)\right\}_{t=1}^T}\left[\log \frac{1}{T} \sum_{t=1}^T \frac{p_{\theta}\left(\bm x|z_t\right) p\left(z_t\right)}{q_\phi\left(z_t| \bm x\right)}\right] =\int  \log \frac{1}{T} \sum_{t=1}^T \frac{p_{\theta}\left(\bm x|z_t\right) p\left(z_t\right)}{q_\phi\left(z_t| \bm x\right)} \cdot \prod_{t=1}^Tp(z_t|\bm x)  dz_t.$
According to Remark \ref{Remark1}, for any posterior $p(z_t|\bm x) \in \mathcal{H}^{(\varrho)}(\mathcal{T})$, there exists a spline function $s(z_t) = \sum_{k=1}^K \gamma_{k} b_k(z_t)$ such that $\sup_{z_t \in \mathcal{T}}|p(z_t|\bm x) - s(z_t)|\leq C H^{-(\varrho+1)}$. We denote the corresponding spline coefficients as $\phi^{\ast}$, and the optimal spline function as $\widehat{s}(z_t)$  whose spline coefficients satisfy $\widehat{\phi} = \argmax   \mathcal{L}_{\text{IWAE}}(\phi)$. Therefore, we have $ \mathcal{L}_{\text{IWAE}}(\widehat{\phi})  \geq \mathcal{L}_{\text{IWAE}}(\phi^\ast).$ We denote $\Delta(z_t) = p(z_t|\bm x) - s^{\ast}(z_t)$ and $\sup_{z_t \in \mathcal{T}}|\Delta(z_t)|\leq C H^{-(\varrho+1)}$. Then, we have
\begin{align}
\notag
    \mathcal{L}_{\text{IWAE}}& (\widehat{\phi})  \geq  \mathcal{L}_{\text{IWAE}}(\phi^{\ast}) = \int \left[ \log \frac{1}{T} \sum_{t=1}^T \frac{p_{\theta}\left(\bm x|z_t\right) p\left(z_t\right)}{p(z_t|\bm x) - \Delta(z_t)} \right] \left[ \prod_{t=1}^Tp(z_t|\bm x) - \Delta(z_t)\right] dz_t \\
    \notag
    & = \int \left\{ \log \left[p_{\theta}(\bm x) + \frac{1}{T} \sum_{t=1}^T \frac{\Delta(z_t) p_{\theta}\left(\bm x|z_t\right) p\left(z_t\right)}{\{p(z_t|\bm x) - \Delta(z_t)\} p(z_t|\bm x)} \right]\right\} \left[ \prod_{t=1}^Tp(z_t|\bm x) - \Delta(z_t)\right] dz_t \\
    \label{EQU:S-ADVI-1}
    & \geq \int \left\{ \log \left[p_{\theta}(\bm x) - C_1 H^{-(\varrho+1)} \right] \right\}\left[ \prod_{t=1}^Tp(z_t|\bm x) - C_2H^{-(\varrho+1)}\right] dz_t  \geq \log p_{\theta}(\bm x) - C_3 H^{-(\varrho+1)},
\end{align}
which yields the results in Lemma \ref{THE:ELBO}.

Next, we consider latent variables with infinite support. For a given $\epsilon$, consider a finite support $\mathcal{T}$ such that $\int_{\mathcal{T}}p\left(z|\bm x\right) dz \geq 1-\epsilon$. The spline-based posterior has a finite support on $\mathcal{T}$. According to the definition of IWAE and the fact that the IWAE has a tighter bound than ELBO, we have
\begin{align*}
    \mathcal{L}_{\text{IWAE}}&(\phi)  =  \mathbb{E}_{\left\{z_t \sim q_\phi(z| \bm x)\right\}_{t=1}^T} \left[\log \frac{1}{T} \sum_{t=1}^T \frac{p_{\theta}\left(\bm x|z_t\right) p\left(z_t\right)}{q_\phi\left(z_t| \bm x\right)}\right] \\
    & =  \int_{\mathcal{T}} \left[ \log \frac{1}{T} \sum_{t=1}^T \frac{p_{\theta}\left(\bm x|z_t\right) p\left(z_t\right)}{q_{\phi}(z_t| \bm x)} \right] \left[ \prod_{t=1}^T q_{\phi}(z_t| \bm x)\right] dz_1 dz_2 \ldots dz_T \\
    & \quad + \int_{\mathbb{R}/\mathcal{T}} \left[ \log \frac{1}{T} \sum_{t=1}^T \frac{p_{\theta}\left(\bm x|z_t\right) p\left(z_t\right)}{q_{\phi}(z_t| \bm x)} \right] \left[ \prod_{t=1}^T q_{\phi}(z_t| \bm x)\right] dz_1 dz_2 \ldots dz_T\\
    & \geq \int_{\mathcal{T}} \left[ \log \frac{p_{\theta}\left(\bm x|z\right) p\left(z\right)}{q_{\phi}(z| \bm x)} \right] q_{\phi}(z| \bm x)  dz + \int_{\mathbb{R}/\mathcal{T}} \left[ \log \frac{p_{\theta}\left(\bm x|z\right) p\left(z\right)}{q_{\phi}(z| \bm x)} \right]  q_{\phi}(z| \bm x) dz = \int_{\mathcal{T}} \left[ \log \frac{p_{\theta}\left(\bm x|z\right) p\left(z\right)}{q_{\phi}(z| \bm x)} \right] q_{\phi}(z| \bm x)  dz  \\
    & = \log p_{\theta}(\bm x) + \int_{\mathcal{T}} \left[ \log \frac{p\left(z|\bm x\right)}{q_{\phi}(z| \bm x)} \right]  q_{\phi}(z| \bm x) dz \geq \log p_{\theta}(\bm x) + \int_{\mathcal{T}} \left[ \log \frac{p^{\textrm{T}}\left(z|\bm x\right)}{q_{\phi}(z| \bm x)} + \log (1-\epsilon) \right]  q_{\phi}(z| \bm x) dz,
\end{align*}
where $p^{\textrm{T}}\left(z|\bm x\right) = (\int_{\mathcal{T}}p\left(z|\bm x\right) dz)^{-1}p\left(z|\bm x\right) $ is the density function for random variable $z|\bm x$ truncated on the interval $\mathcal{T}$. Notice that when $\mathcal{T}$ is properly chosen and we can apply the conclusion from (\ref{EQU:S-ADVI-1}) 
\[
\log p_{\theta}(\bm x) + \int_{\mathcal{T}} \left[ \log \frac{p^{\textrm{T}}\left(z|\bm x\right)}{q_{\phi}(z| \bm x)} + \log (1-\epsilon) \right]  q_{\phi}(z| \bm x) dz \geq \log p_{\theta}(\bm x) - C_3 H^{-(\varrho+1)} - C_3\epsilon,
\]
which yields the results in Lemma \ref{THE:ELBO}.

It is straightforward to extend the above results to the multivariate case. If Assumption (A2) holds, then we have 
\begin{align*}
  \int_{\mathcal{T}_1 \cdots \mathcal{T}_J} \left[ \log \frac{p \left(\bm z|\bm x\right)}{q_{\phi}(\bm z| \bm x)}\right]  q_{\phi}(\bm z| \bm x) d\bm z & = \int_{\mathcal{T}_1 \cdots \mathcal{T}_J} \sum_{j=1}^J \left[ \log \frac{p \left(z_j|\bm x\right)}{q_{\phi}(z_j| \bm x)}\right] \left[\prod_{j=1}^J q_{\phi}(z_j| \bm x) \right] dz_1 \cdots dz_J \\
  &  = \sum_{j=1}^J  \int_{\mathcal{T}_j} \left[ \log \frac{p \left(z_j|\bm x\right)}{q_{\phi}(z_j| \bm x)}\right] q_{\phi}(z_j| \bm x) dz_j.
\end{align*}
Similar to the results in (\ref{EQU:S-ADVI-1}), we can obtain the lower bound of IWAE based on the S-ADVI is $\log p_{\theta}(\bm x) - C J H^{-(\varrho+1)} - J \epsilon$.
\end{proof}

According to the results in Lemma \ref{THE:ELBO}, we can further quantify the variational approximation error with respect to the class defined in (\ref{EQU:S-ADVI}). 
\begin{proof}[Proof of Theorem \ref{THE:Approximation}]
According to Remark \ref{Remark1} and similar to the Proof of Lemma \ref{THE:ELBO}, there exists  $q_{\phi}(\bm z| \bm x)$ such that $\sup_{\bm z}|p(\bm z|\bm x) - q_{\phi}(\bm z| \bm x)|\leq C JH^{-(\varrho+1)}$. When $H$ goes to infinity, $JH^{-(\varrho+1)}$ goes to zero. One can also infer that $|p(\bm z|\bm x) - q_{\widehat{\phi}}(\bm z| \bm x)|\leq C JH^{-(\varrho+1)}$ almost everywhere on $\mathcal{T}_1 \times \cdots \times \mathcal{T}_J$. Then, the following property holds that is, 
\begin{align*}
   & \mathcal{L}_{\text {IWAE }}(\widehat{\phi}) = \mathbb{E}_{\left\{q_{\widehat{\phi}}( \bm z_t| \bm x)\right\}_{t=1}^T}\left[\log \frac{1}{T} \sum_{t=1}^T \frac{p\left(\bm z_t|\bm x\right)}{q_{\widehat{\phi}}\left(\bm z_t| \bm x\right)}\right] + \log p_{\theta}(\bm x)  \\
   & = \mathbb{E}_{\left\{q_{\widehat{\phi}}( \bm z_t| \bm x)\right\}_{t=1}^T}\left\{\log  \left[1 + \frac{1}{T} \sum_{t=1}^T \frac{p\left(\bm z_t|\bm x\right) - q_{\widehat{\phi}}\left(\bm z_t| \bm x\right)}{q_{\widehat{\phi}}\left(\bm z_t| \bm x\right)}\right]\right\} + \log p_{\theta}(\bm x) \\
   & = \mathbb{E}_{\left\{q_{\widehat{\phi}}( \bm z_t| \bm x)\right\}_{t=1}^T}\left\{\frac{1}{T} \sum_{t=1}^T \frac{p\left(\bm z_t|\bm x\right) - q_{\widehat{\phi}}\left(\bm z_t| \bm x\right)}{q_{\widehat{\phi}}\left(\bm z_t| \bm x\right)} + o\left[JH^{-(\varrho+1)}\right]\right\} + \log p_{\theta}(\bm x)\\
   & = \mathbb{E}_{q_{\widehat{\phi}}( \bm z| \bm x)}\left\{\frac{p\left(\bm z|\bm x\right) - q_{\widehat{\phi}}\left(\bm z| \bm x\right)}{q_{\widehat{\phi}}\left(\bm z| \bm x\right)} + o\left[JH^{-(\varrho+1)}\right] \right\} + \log p_{\theta}(\bm x) \\
   & = \mathbb{E}_{q_{\widehat{\phi}}( \bm z| \bm x)}\left\{\log \frac{p\left(\bm z|\bm x\right)}{q_{\widehat{\phi}}\left(\bm z| \bm x\right)}+ o\left[JH^{-(\varrho+1)}\right]\right\} + \log p_{\theta}(\bm x) \\
   & = -D_{\textrm{KL}}\left[q_{\widehat{\phi}}\left( \bm z| \bm x\right)||p\left(\bm z|\bm x\right)\right] + \log p_{\theta}(\bm x) + o\left[JH^{-(\varrho+1)}\right]
\end{align*}
The result in Lemma \ref{THE:ELBO} implies the KL divergence of the spline density approximation from the true posterior is in the order of $JH^{-(\varrho+1)} + J \epsilon$.  
\end{proof}

\subsection{Proof of Theorem \ref{THE:Approximation-Regression}}
\label{SEC:THE:Approximation-Regression-Proof}

\begin{lemma}[Lemma 6.3 \citep{csiszar2006context}]
\label{LEM:Kl-L2}
    If $p$ and $q$ are probability densities both supported on a bounded interval $\mathcal{T}$, then we have the KL divergence between probability densities satisfies that $D_{\mathrm{KL}}(p||q) \leq \frac{1}{\inf _{x \in \mathcal{T}} q(x)}\|p-q\|_2^2.$
\end{lemma}
\begin{proof} Notice that
    $$
\begin{aligned}
D_{\mathrm{KL}}(p||q) & =\int_{\mathcal{T}} p(x) \log \frac{p(x)}{q(x)} \mathrm{d} x \leq \int_{\mathcal{T}}p(x)\left\{\frac{p(x)}{q(x)}-1\right\} \mathrm{d} x  =\int_{\mathcal{T}}\frac{\{p(x)-q(x)\}^2}{q(x)} \mathrm{d} x
\end{aligned}
$$
from which the claim follows.
\end{proof}

Here we provide the proof sketch of Theorem \ref{THE:Approximation-Regression}. Our proof is based on the assumption when the observed data points $\boldsymbol x$ and $\boldsymbol x'$ are close enough, the corresponding posteriors $p(\bm z|\boldsymbol x)$ and $p(\bm z|\boldsymbol x')$ are similar to each other. For any given posterior $p(\bm z|\boldsymbol x)$ under satisfying Assumption (A2), we can identify the density function based on spline approximation $\prod_{j=1}^J \sum_{k=1}^K \gamma^\ast_{jk}(\boldsymbol x) b_{k, \mathcal{T}^\ast_j}(z_j)$ with $\mathcal{T}^\ast_j = [\mu^\ast_j(\boldsymbol x), \mu^\ast_j(\boldsymbol x) + \sigma^\ast_j(\boldsymbol x)]$ close to $p(\bm z|\boldsymbol x)$ with differences bounded by $JH^{-(\varrho+1)}$. Under some mild assumptions, $\gamma^\ast_{jk}(\boldsymbol x)$, $k=1,\ldots, K$, $\mu^\ast_j(\boldsymbol x)$, and $\sigma^\ast_j(\boldsymbol x)$ can be well approximated by nonparametric regression, such as the deep neural network. Combining the results in Theorem \ref{THE:Approximation}, we can further obtain the KL divergence between the proposed S-ADVI estimator and the true posterior.

\begin{proof}[Proof of Theorem \ref{THE:Approximation-Regression}]

    In the following, we denote that the unknown parameters as $\phi = \{\mu_j, \sigma_j, \gamma_{jk}, j = 1, \ldots, J, k = 1, \ldots, K\}$ and $\widehat{\phi}(\bm x) = \argmax_{\phi}  \sum_{i=1}^n \mathcal{L}_{\text{ELBO}}\{\phi(\bm x_i)\}$. We notice that $ \sum_{i=1}^n \mathcal{L}_{\text{ELBO}}\{\widehat{\phi}(\bm x_i)\} = \sum_{i=1}^n \log p_{\theta}(\bm x_i)-\sum_{i=1}^n D_{KL}\left\{q_{\widehat{\phi}(\bm x_i)}(\bm z)~||~p(\bm z | \bm x_i)\right\}$, we can infer to the average of the KL divergence between $q_{\widehat{\phi}(\bm x_i)}(\bm z)$ and $p(\bm z | \bm x_i)$ equaling to $n^{-1}\sum_{i=1}^n \log p_{\theta}(\bm x_i) - n^{-1}\sum_{i=1}^n \mathcal{L}_{\text{ELBO}}\{\widehat{\phi}(\bm x_i)\}$. In addition, for a given $\xi$, consider a finite support $\mathcal{T}_j^{\ast}=[\mu_j^\ast(\bm x), \mu_j^\ast(\bm x) + \sigma_j^\ast(\bm x)]$ such that $\int_{\mathcal{T}_j^{\ast}}p\left(z_j|\bm x\right) dz_j \geq 1-\xi$. For a specific posterior $p(\bm z|\bm x)$, the optimal parameters are $\phi^\ast(\bm x) = \left\{\mu^\ast_j(\bm x), \sigma^\ast_j(\bm x), \gamma^\ast_{jk}(\bm x), j = 1, \ldots, J, k=1, \ldots, K\right\}$, where $\gamma^\ast_{jk}(\bm x)$'s are spline coefficients satisfying that  $\sup_{z_j \in \mathcal{T}_j^\ast}|p(z_j|\bm x) - s^\ast(z_j;\bm x)|\leq C H^{-(\varrho+1)}$, where $s^\ast(z_j;\bm x) = \sum_{k=1}^K \gamma^\ast_{jk}(\bm x) b_k(z_j)$.

    Next, we denote the estimators of the optimal parameters generated from nonparametric regression, such as the deep neural network, as $\widetilde{\phi}(\bm x) = \left\{\widetilde{\mu}_j(\bm x), \widetilde{\sigma}_j(\bm x), \widetilde{\gamma}_{jk}(\bm x), j = 1, \ldots, J, k=1, \ldots, K\right\}$. When the observed datapoints $\boldsymbol x$ and $\boldsymbol x'$ are close enough, the corresponding posteriors $p(\bm z|\boldsymbol x)$ and $p(\bm z|\boldsymbol x')$ are close, so do the corresponding optimal parameters $\phi^\ast(\bm x)$ and $\phi^\ast(\bm x')$.  Then, we have 
    \begin{align*}
    \frac{1}{n}& \sum_{i=1}^n D_{\mathrm{KL}}\left\{q_{\widehat{\phi}(\bm x_i)}(\bm z)~||~p(\bm z | \bm x_i)\right\} \\
      & = \frac{1}{n}\sum_{i=1}^n \log p_{\theta}(\bm x_i) - \frac{1}{n}\sum_{i=1}^n \mathcal{L}_{\text{ELBO}}\{\widehat{\phi}(\bm x_i)\}  \leq \frac{1}{n}\sum_{i=1}^n \log p_{\theta}(\bm x_i) - \frac{1}{n}\sum_{i=1}^n \mathcal{L}_{\text{ELBO}}\{\widetilde{\phi}(\bm x_i)\}\\
      & = \frac{1}{n} \sum_{i=1}^n D_{\mathrm{KL}}\left\{q_{\widetilde{\phi}(\bm x_i)}(\bm z)~||~p(\bm z | \bm x_i)\right\}. 
    \end{align*}

    According to Lemma \ref{LEM:Kl-L2}, we have 
    \begin{align*}
    \frac{1}{n}&\sum_{i=1}^nD_{\mathrm{KL}}  \left\{q_{\widetilde{\phi}(\bm x_i)}(\bm z)~||~p(\bm z | \bm x_i)\right\} \\
    & = \frac{1}{n}\sum_{i=1}^n \int  q_{\widetilde{\phi}(\bm x_i)}(\bm z) \log \frac{q_{\widetilde{\phi}(\bm x_i)}(\bm z)}{p(\bm z | \bm x_i)} \mathrm{d} \bm z = \sum_{i=1}^n \sum_{j=1}^J \int  q_{\widetilde{\phi}(\bm x_i)}(z_j) \log \frac{q_{\widetilde{\phi}(\bm x_i)}(z_j)}{p(z_j | \bm x_i)} \mathrm{d} z_j \\
    & \leq \frac{1}{n}\sum_{i=1}^n \sum_{j=1}^J \frac{1}{\inf_{z_j \in \widetilde{\mathcal{T}}_j} p(z_j|\bm x_i)} \|p(z_j| \bm x_i) - q_{\widetilde{\phi}(\bm x_i)}(z_j)\|^2  \\
    & \leq \frac{1}{n}\sum_{i=1}^n \sum_{j=1}^J \frac{2}{\inf_{z_j \in \widetilde{\mathcal{T}}_j} p(z_j |\bm x_i)} \left\{\|p(z_j| \bm x_i) - q_{\phi^{\ast} (\bm x_i)}(z_j)\|^2 + \|q_{\widetilde{\phi}(\bm x_i)}(z_j) - q_{\phi^{\ast}(\bm x_i)}(z_j)\|^2\right\}\\
    & \leq C_1 (J H^{-2(\varrho +1)} + J \epsilon^2) + C_2 J H^2  \Delta^2,
    \end{align*}
    where $\widetilde{\mathcal{T}}_j$ is the support of density function $q_{\widetilde{\phi}(\bm x_i)}(z_j)$ and $\widetilde{\mathcal{T}}_j = [\widetilde{\mu}_j(\bm x), \widetilde{\mu}_j(\bm x) + \widetilde{\sigma}_j(\bm x)]$. According to the properties of spline basis functions, we have $n^{-1}\sum_{i=1}^n\|q_{\widetilde{\phi}(\bm x_i)}(z_j) - q_{\phi^{\ast}(\bm x_i)}(z_j)\|^2 \leq C H \sum_{j=1}^J \sum_{k=1}^K \{\gamma^\ast_{jk}(\bm x_i) - \widetilde{\gamma}_{jk}(\bm x_i)\}^2 = O(H^2  \Delta^2)$.

\end{proof}

\section{\uppercase{Implementation Details}}

\setcounter{equation}{0}

\subsection{Spline Roughness Penalty}
\label{SEC:Implementation}

We introduce the details of constructing penalty $\mathbf{P}$ matrix. Following the definition of penalty matrix $\mathbf{P}$ in the Section \ref{SEC:modelestimation}, for any spline polynomials $s(t) = \sum_{k=1}^K \gamma_kb_k(t)$, we can further implement as
\begin{align}
\mathcal{E}(s)
&= \int_{\mathcal{T}} \left\{\sum_{k=1}^K \gamma_k b_k''(t)\right\}^2dt = \sum_{k=1}^K \sum_{k'=1}^K \gamma_{k} \gamma_{k'}  \int_{\mathcal{T}} b_k''(t)b_{k'}''(t)d t = \bs{\gamma}^{\top} \mathbf{P} \bs{\gamma}, \label{EQN:EsTP}
\end{align}
where $\mathbf{P}$  is a $K \times K$ matrix with entries $\int_{\mathcal{T}} b_k''(t)b_{k'}''(t)d t$. Denote that $\bm b''(t)$ as the vector of the second-order derivatives of the spline basis $\bm b(t)$ and $\bm b''(t) = \mathbf{D}^{(2)} \bm b(t) = \mathbf{M}_1 \mathbf{M}_2 \bm b_{\varrho-2}(t)$ where $\bm b_{\varrho-2}(t)$ is the vector of spline basis with degree $\varrho -2$ and 
$$
\mathbf{M}_\ell=(\varrho+1-\ell)\left(\begin{array}{ccccc}
\frac{-1}{\upsilon_1-\upsilon_{-\varrho+\ell}} & 0 & 0 & \cdots 0 & 0 \\
\frac{1}{\upsilon_1-\upsilon_{-\varrho+\ell}} & \frac{-1}{\upsilon_2-\upsilon_{1-\varrho+\ell}} & 0 & \cdots 0 & 0 \\
0 & \frac{1}{\upsilon_2-\upsilon_{1-\varrho+\ell}} & \frac{-1}{\upsilon_3-\upsilon_{2-\varrho+\ell}} & \cdots 0 & 0 \\
\vdots & \vdots & \vdots & \ddots & \vdots \\
0 & 0 & 0 & \cdots & \frac{1}{\upsilon_{H+\varrho +1- \ell}-\upsilon_{H+\varrho +1}}
\end{array}\right) \text {, for } \ell = 1, 2.
$$

\subsection{Details of Reparameterization Trick} \label{apdx:reparam_trick}

\begin{lemma}
    \label{Prop:density}
    If $\bm z = g(\bm \epsilon)$, $\bm g$ is a monotone function, then the density function of $\bm z$ is 
    $\bm f_{\bm z}(\bm z) = \bm f_{\bm \epsilon} \{\bm g^{-1}(\bm z)\} \left|\frac{\partial \bm z }{\partial \bm \epsilon} \right|,$
    where $\left|\frac{\partial \bm z }{\partial \bm \epsilon} \right|$ is the determinate of the Jacob matrix $\frac{\partial \bm z }{\partial \bm \epsilon}$. 
\end{lemma}
In the following, we derive that 
\[
\mathcal{L}_{\text{IWAE}}(\phi) =  \mathbb{E}_{\left\{\bm \epsilon_t\right\}_{t=1}^T}\left[\log \frac{1}{T} \sum_{t=1}^T \frac{p_{\theta}\left\{\bm x,\bm \mu(\bm x) + \bm \sigma (\bm x) \cdot \bm \epsilon_t \right\}}{\prod_{j=1}^J\left\{\sum_{k=1}^K \gamma_{jk}(\bm x) b_k\left(\epsilon_{jt}\right) \right\}}\right]  + \sum_{j=1}^J\log \sigma_j(\bm x).
\]
Note that $\bm z_{j} = \mu_j(\bm x) + \sigma_j (\bm x) \epsilon_{j}$. Applying the conclusion in Lemma \ref{Prop:density}, the function $q_{\phi}(z_{j}|\bm x)$ is 
\begin{equation}
\label{EQU:density}
    \frac{1}{\sigma_{j}(\bm x)}\sum_{k=1}^K \gamma_{jk}(\bm x) b_k\left\{\frac{z_{j}-\mu_j(\bm x)}{\sigma_j(\bm x)}\right\}.
\end{equation} 
Plugging in (\ref{EQU:density}) to $\mathcal{L}_{\text{IWAE}}(\phi)$, we have
\begin{align*}
   & \mathcal{L}_{\text{IWAE}}(\phi) = \mathbb{E}_{\left\{q_\phi( \bm z_t| \bm x)\right\}_{t=1}^T}\left[\log \frac{1}{T} \sum_{t=1}^T \frac{p_{\theta}\left(\bm x| \bm z_t\right) p\left(\bm z_t\right)}{q_\phi\left(\bm z_t| \bm x\right)}\right]\\
    &~= \int \left[\log \frac{1}{T} \sum_{t=1}^T \frac{p_{\theta}\left(\bm x, \bm z_t\right)}{q_{\phi}(\bm z_t| \bm x)} \right] \left[ \prod_{t=1}^T q_{\phi}(\bm z_t| \bm x) \right] d\bm z_t  \\
    &~= \int \left[\log \frac{1}{T} \sum_{t=1}^T \frac{p_{\theta}\left(\bm x, \bm z_t\right)}{q_{\phi}(\bm z_t| \bm x)} \right]  \left\{\prod_{t=1}^T  \prod_{j=1}^J \frac{1}{\sigma_{j}(\bm x)}(\bm x)\sum_{k=1}^K \gamma_{jk}(\bm x) b_k\left[\frac{z_{j}-\mu_j(\bm x)}{\sigma_j(\bm x)}\right] \right\} d\bm z_t \\
    &~= \int \left[\log \frac{1}{T} \sum_{t=1}^T \frac{p_{\theta}\left(\bm x, \bm \mu(\bm x) \!+\! \bm \sigma (\bm x) \cdot \bm \epsilon_t \right)}{\prod_{j=1}^{J}\frac{1}{\sigma_{j}(\bm x)}\sum_{k=1}^K \gamma_{jk}(\bm x) b_k(\epsilon_{jt})} \right]\left\{ \prod_{t=1}^T  \prod_{j=1}^J \left[\sum_{k=1}^K \gamma_{jk}(\bm x) b_k(\epsilon_{jt}) \right] \right\} d\bm \epsilon_t \\
    &~= \int \left[\log \frac{1}{T} \sum_{t=1}^T \frac{p_{\theta}\left(\bm x, \bm \mu(\bm x) + \bm \sigma (\bm x) \cdot \bm \epsilon_t \right)}{\prod_{j=1}^{J}\sum_{k=1}^K \gamma_{jk}(\bm x) b_k(\epsilon_{jt})} \right] \left\{ \prod_{t=1}^T  \prod_{j=1}^J \left[\sum_{k=1}^K \gamma_{jk}(\bm x) b_k(\epsilon_{jt})\right] \right\} d\bm \epsilon_t \!\!+ \!\!\sum_{j=1}^J \log \sigma_j(\bm x) \\
    &~=\mathbb{E}_{\left\{\bm \epsilon_t\right\}_{t=1}^T}\left[\log \frac{1}{T} \sum_{t=1}^T \frac{p_{\theta}\left\{\bm x,\bm \mu(\bm x) + \bm \sigma (\bm x) \cdot \bm \epsilon_t \right\}}{\prod_{j=1}^J\left\{\sum_{k=1}^K \gamma_{jt}(\bm x) b_k\left(\epsilon_{jt}\right) \right\}}\right] + \sum_{j=1}^J \log \sigma_j(\bm x).
\end{align*}

\subsection{Posterior Collapse}
\label{SEC:Collapse}

In generative models, posterior collapse means that the posterior of the latent variables equals their prior, that is, $p(\bm z| \bm x)=p(\bm z)$ \citep{Wang:Blei:Cunningham:2021}. 
ADVI with Gaussian mixture approximation (GM-ADVI) is one approach to generate flexible posterior approximation \citep{morningstar2021automatic}. Nonetheless, it may suffer from posterior collapse. An example of the GM-ADVI collapse is given as follows. The latent variable $\bm z$ is generated from a mixture distribution: 
$p\left(u\right)=\text{Categorical}(1 / M),\, p\left(\bm z|u\right)= N \left(\bm \mu_{u}, \mathbf{\Sigma}_{u}\right)$, where $u$ takes values from 1 to $M$, $\bm \mu_m$ 's are $d$-dimensional, and $\mathbf{\Sigma}_m$ are $d \times d$-dimensional covariance matrix. The observed data $\bm x$ follows  $p\left(\bm x|\bm z; f, \sigma\right) = N \left(f\left(\bm z\right), \sigma^2\mathbf{I}\right)$. Then, the marginal distribution of $\bm z$ is $1/M \sum_{m=1}^M  N \left(\bm \mu_{m}, \mathbf{\Sigma}_{m}\right)$, which can be well approximated via GM-ADVI. We are interested in estimating the latent class $u$ and the conditional distribution of $\bm z| u$ for a given data point $\bm x$. In this case, if we have $\bm z| u = 1/M \sum_{m=1}^M  N \left(\bm \mu_{m}, \mathbf{\Sigma}_{m}\right)$ for all the classes, we can still fully capture the mixture distribution of the data. However, all the $M$ mixture components are the same, and thus the latent variable is non-identifiable. In contrast to the above GM-ADVI example, the proposed S-ADVI is based on a set of pre-specified spline density functions and is constrained not to capture the marginal distribution of $\bm z$, preventing the mixture components from shrinking to the prior distribution. Therefore, posterior collapse can be naturally avoided via spline density functions.

The phenomenon that the posterior is approximately (as opposed to exactly) equal to the prior can also lead to posterior collapse. In S-ADVI, one can easily control the phenomena by adding regularity constraints on spline coefficients. Specifically, the regularity constraint can be written as $\{\|\bm \gamma_j (\bm x)- \bm \gamma_0\|\}^{-1}$, where $\bm \gamma_j (\bm x)$ is the spline coefficient vector for $j$-th latent variable and $\bm \gamma_0$ is the spline coefficient vector whose corresponding density function is closest to the prior. The regularity constraint prevents the posteriors from shrinking to the prior distribution.




\section{\uppercase{Additional Experiment Results}}

\subsection{Additional Results in the Posterior Approximation}

In Figure \ref{fig:all_cases} and Table \ref{tab:complete_posterior_infer_res}, we present additional results showcasing posterior approximations using S-ADVI, Gaussian-ADVI, and GM-ADVI for Cases 1 through 5, considering varying values of $x$. Our findings demonstrate that S-ADVI outperforms other methods in most scenarios. Also, in Figure \ref{fig:case_convergence}, we compare the convergence of S-ADVI against GM-ADVI for the five cases.

\begin{figure}[!tb]
\centering
\includegraphics[width=1.05\textwidth]{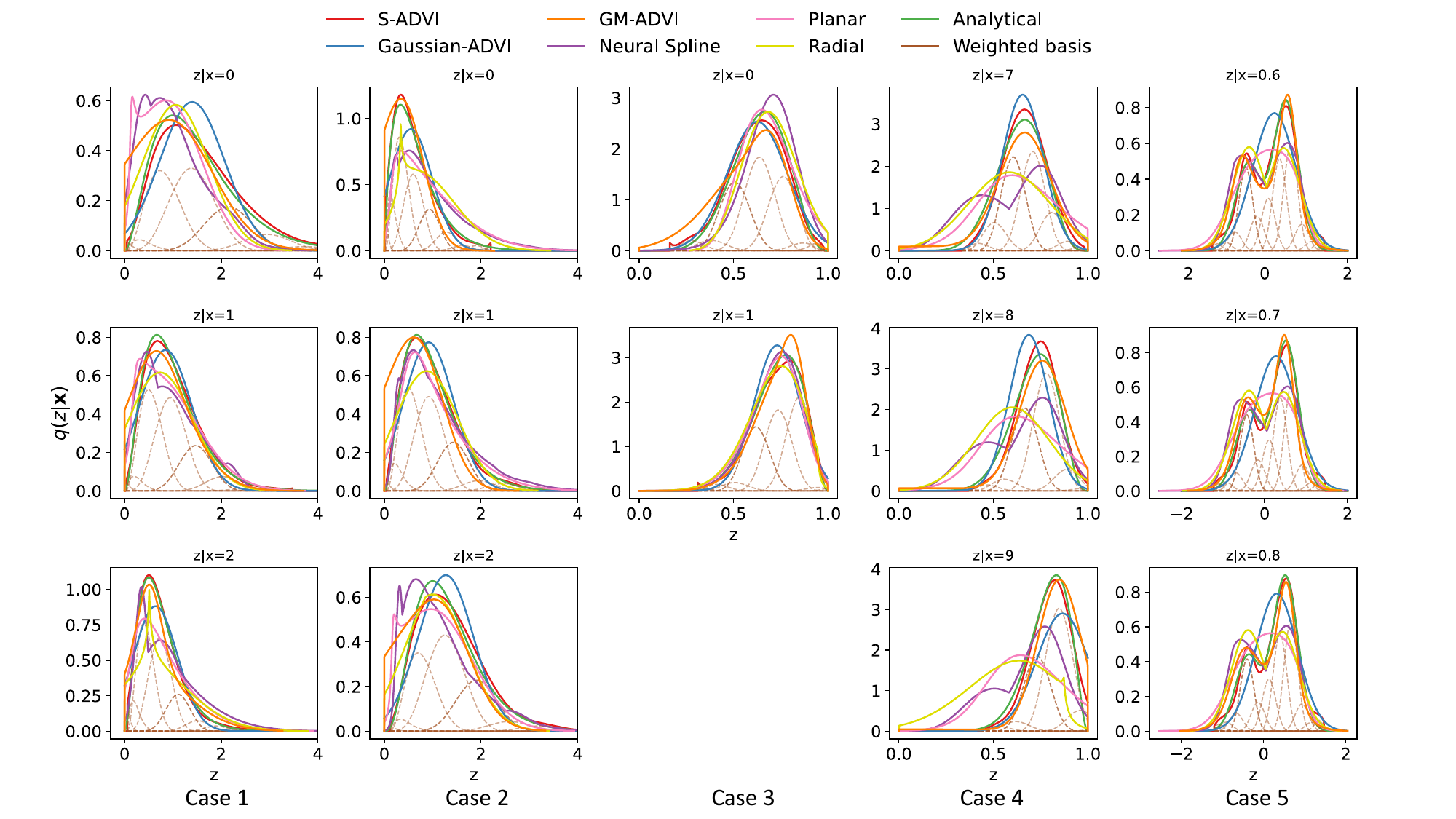}
\caption{\capitalisewords{Visualization of approximated posterior based on the S-ADVI, Gaussian-ADVI, GM-ADVI, Neural Spline, Planar, and Radial for Cases 1 -- 5}}
\label{fig:all_cases}
\end{figure}


\begin{figure}[!htb] 
	\centering
    \subfloat[Case 1]{\includegraphics[width=.1925\textwidth]{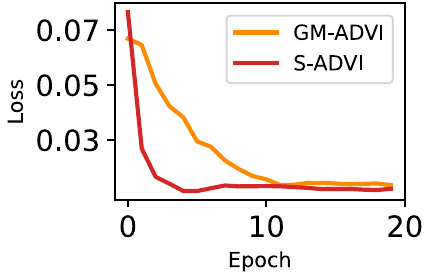}}
  	\subfloat[Case 2]{\includegraphics[width=.1925\textwidth]{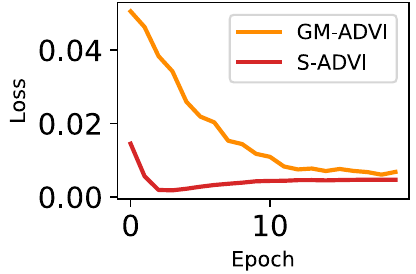}}
	\subfloat[Case 3]{\includegraphics[width=.1925\textwidth]{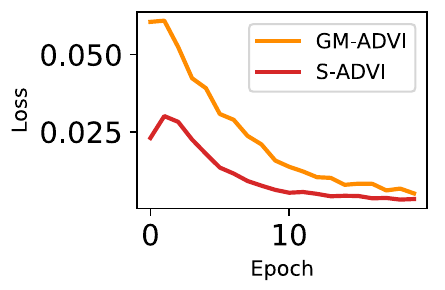}}
    \subfloat[Case 4]{\includegraphics[width=.1925\textwidth]{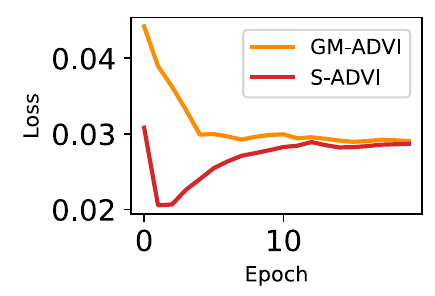}}
    \subfloat[Case 5]{\includegraphics[width=.1925\textwidth]{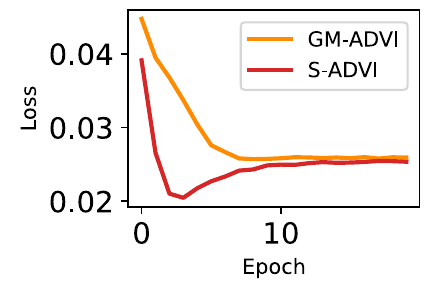}}
	\caption{\capitalisewords{Comparison of convergence between S-ADVI and GM-ADVI}} \label{fig:case_convergence}
\end{figure}

\begin{table}[htbp]
  \centering
    \caption{\capitalisewords{Mean and Standard Deviation (in bracket) of Root Integrated Squared Error in Posterior Approximation Based on Gaussian-ADVI, GM-ADVI, Neural Spline, Planar, Radial, and S-ADVI}}\label{tab:complete_posterior_infer_res}%
    \vspace{1ex}
  \resizebox{1.0\textwidth}{!}{ 
    \begin{tabular}{lcccccccccccccccccc}
    \hline \hline
     \multirow{2}{*}{Method}     & \multicolumn{3}{c}{Case 1} &       & \multicolumn{3}{c}{Case 2} &       & \multicolumn{2}{c}{Case 3} &       & \multicolumn{3}{c}{Case 4} &       & \multicolumn{3}{c}{Case 5} \\
\cmidrule{2-4}\cmidrule{6-8}\cmidrule{10-11}\cmidrule{13-15}\cmidrule{17-19}          & \multicolumn{1}{l}{$x=0$} & \multicolumn{1}{l}{$x=1$} & \multicolumn{1}{l}{$x=2$} &       & \multicolumn{1}{l}{$x=0$} & \multicolumn{1}{l}{$x=1$} & \multicolumn{1}{l}{$x=2$} &       & \multicolumn{1}{l}{$x=0$} & \multicolumn{1}{l}{$x=1$} &       & \multicolumn{1}{l}{$x=7$} & \multicolumn{1}{l}{$x=8$} & \multicolumn{1}{l}{$x=9$} &       & \multicolumn{1}{l}{$x=0.6$} & \multicolumn{1}{l}{$x=0.7$} & \multicolumn{1}{l}{$x=0.8$} \\
    \midrule
    \multirow{2}{*}{\makecell{Gaussian-ADVI} } &0.295 & 0.408 & 0.713  &       & 0.268  & 0.239 & 0.251 &       & 0.630 & 0.529&       & 0.631  & 0.876 & 1.14  &       & 0.249  & 0.246 & 0.243  \\
    &(0.148) &  (0.274) & (0.381) &       &  (0.057) &  (0.044) &  (0.064) &       &  (0.239) & (0.229)&       &  (0.386) & (0.394) &  (0.431) &       & (0.014) & (0.013) & (0.014)\\
    \hline
    \multirow{2}{*}{GM-ADVI} & 0.206 & 0.353 & 0.174 &       & 0.237 & 0.211& 0.226 &       &  0.403  & 0.171 &       & 0.395  & 0.317 & 0.487 &       & 0.123 & 0.136 &  0.137 \\
    & (0.103) & (0.152)& (0.095) &       & (0.031) & (0.026) & (0.040) &       &  (0.115) & (0.121) &       & (0.107) & (0.096) & (0.134) &       & (0.017) & (0.019) &   (0.020) \\
    \hline
    \multirow{2}{*}{Neural Spline } & 0.402  & 0.295 & 0.335  &       & 0.356 & 0.158  & 0.324  &       & 0.528  & 0.231 &       & 0.956  & 1.059 & 1.299  &       & 0.476  & 0.486  &  0.497  \\
    &  (0.055) &  (0.055)&  (0.053) &       &  (0.048) &  (0.032) &  (0.048) &       &  (0.077) &(0.082)&       &  (0.075) & (0.147) & (0.197) &       &  (0.007) &  (0.007) &   (0.007) \\ \hline 
     \multirow{2}{*}{Planar} & 0.367  &  0.279 & 0.345  &       & 0.369  &  0.182  & 0.316  &       & 0.384  & 0.390 &       & 1.043  & 1.209  & 1.400  &       & 0.470  & 0.480  &  0.491  \\
     & (0.046) &  (0.055)& (0.051) &       & (0.045) &   (0.031) &  (0.052) &       &  (0.099) &  (0.105)&       &  (0.241) & (0.160) &  (0.123) &       &  (0.009) &  (0.010) &  (0.010) \\ \hline 
      \multirow{2}{*}{Radial}& 0.274  &  0.258 & 0.419  &       & 0.499  &   0.212  & 0.198  &       & 0.440  & 0.283 &       & 0.927  & 1.174  & 1.480  &       & 0.484) & 0.495  &  0.504  \\
      & (0.053) &   (0.058)&  (0.045) &       &  (0.043) &    (0.026) &  (0.038) &       & (0.102) & (0.112)&       & (0.155) & (0.156) & (0.183) &       & (0.010) &  (0.010) &   (0.011) \\ \hline 
    \multirow{2}{*}{S-ADVI} & \textbf{0.094} & \textbf{0.086} & \textbf{0.088} &       & \textbf{0.088} & \textbf{0.054} & \textbf{0.101} &       & \textbf{0.211} & \textbf{0.146} &       & \textbf{0.310} & \textbf{0.323} & \textbf{0.371} &       & \textbf{0.101} & \textbf{0.099} & \textbf{0.097} \\
    & {(0.029)} & { (0.046)} & {(0.054)} &       & {(0.027)} & {(0.020)} & {(0.028)} &       & {(0.070)} & {(0.041)} &       & {(0.080)} & {(0.101)} & {(0.141)} &       & {(0.024)} & {(0.025)} & {(0.025)} \\
    \hline \hline
    \end{tabular}%
}

\end{table}%

\subsection{The Effect of Hyper-tuning Parameters in Posterior Approximation}
We aim to evaluate the effects of hyperparameters, such as the impact of roughness penalty parameters, temperature decay, and the number of random samples in IWAE. See Figure \ref{fig:tunning} for an illustration of hyperparameter effects. For both Case 1 and Case 2 in Figure \ref{fig:tunning}, we notice that the spline roughness penalty can balance the bias and variance. When set to be small, the estimated curve is more wiggling than the scenario with a large roughness penalty. When the roughness penalty is large, the model is overly smooth. The model is not very sensitive to the speed of decay of temperature (Temp.Decay). But model fitness is less than ideal when it is too large or small for different reasons. For example, from Figure 2, both Case 1 and Case 2 show good performances except when Temp.Decay = 1 or 15. When it is too small, concrete distribution cannot approximate categorical distribution. When it is too large, in the training process, some components may not be fully explored. For the IWAE random samples, the results in Figure \ref{fig:tunning} are consistent with previous works. In Figure \ref{fig:tunning} (a), for the Gamma posterior distribution, approximation performance varies little with the increasing number of random samples. However, when the posterior distribution is multimodal, the approximation is beneficial from the IWAE random samples. 

\begin{figure*}[!htb]
\centering
\begin{tabular}{@{}c@{}}
  
\includegraphics[width=0.9\textwidth]{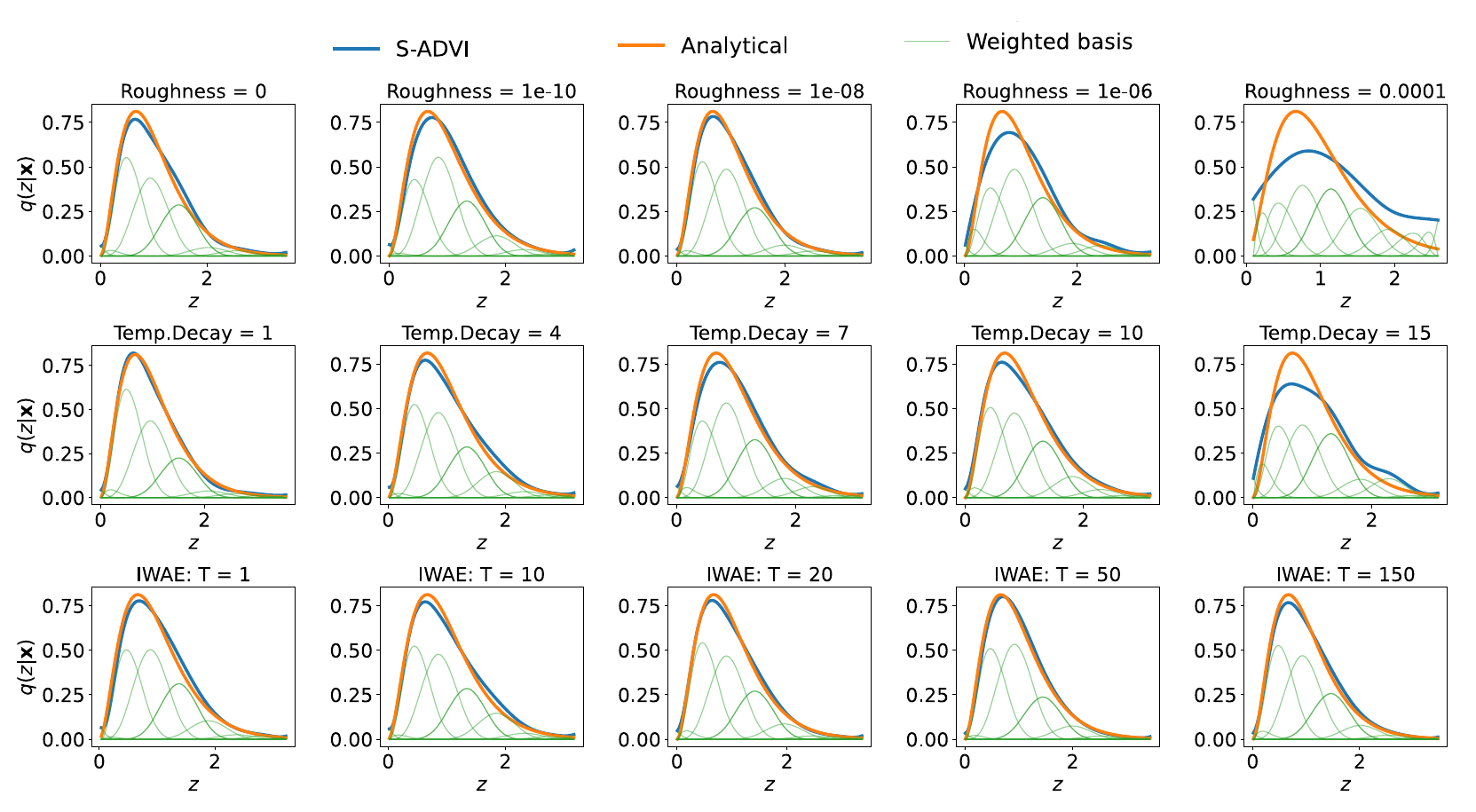}\\
\small (a) Case 1: Gamma-Exponential
\\
\\
\end{tabular}
 \vspace{\floatsep}
\begin{tabular}{@{}c@{}}
\includegraphics[width=0.9\textwidth]{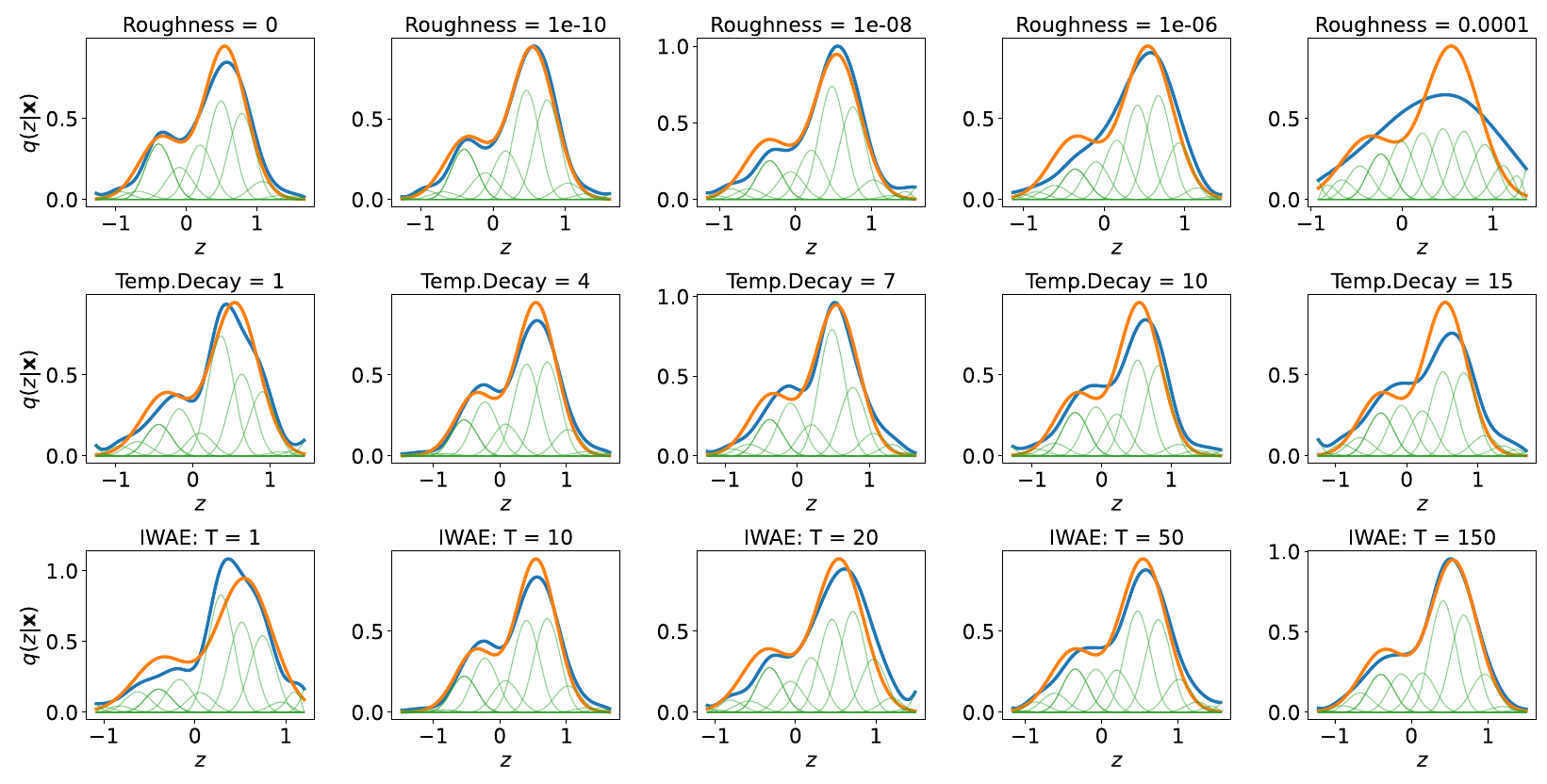}
\\
\small (b) Case 2: Gaussian Mixture
\end{tabular}
\caption{\capitalisewords{Hyperparameter effects on S-ADVI}}
\label{fig:tunning}
\end{figure*}

\subsection{Additional Results for Single Column Classification}

Figure \ref{fig:singlecolumn} represents an illustration example of an input sample. The black column is the input for image classification. We train the model with $H = \{1, 3, 5, 7, 9\}$ and $J = \{2, 4, 6, 8\}$. We set $v_0 = 0$ and $v_{H+1} = 1$, choose $v_h$ with equal space for $h=1,\ldots, H$, and set $T = 10$ for the IWAE. The encoder is initialized as an MLP with two layers of 512 and 256 hidden nodes and ELU activation function \citep{clevert2015fast}. The output layer predicts the parameters for distribution over $J$ latent variables. All models are tuned with 10-fold cross-validation and trained for 50 epochs using Adam optimizer with a learning rate 0.001 decayed by 5\% for every epoch. We use the $\beta = 0.05$ penalty on the KL divergence term. We choose $\eta=4,\, \Lambda_0=1,\, \Lambda_1=0.05$ in annealing such that the temperature decreases from 1 to 0.05 in 15 epochs. Figure \ref{fig:performace-supp} presents the performance comparison for single-column classification for the MNIST dataset with a range of latent variables $J$ and interior knots $H$. Error bars denote the standard error. Figure \ref{fig:fminst_loss} demonstrates the convergence of S-ADVI for MINST classification.

\begin{figure}[h]
    \centering
    \includegraphics[width=0.2\textwidth]{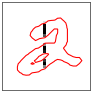}
    \caption{\capitalisewords{An illustration of a single column of an input sample, where The black column is the input for image classification}}
    \label{fig:singlecolumn}
\end{figure}

\begin{figure*}[!tb] 
    \centering
    \includegraphics[width=0.92\textwidth]{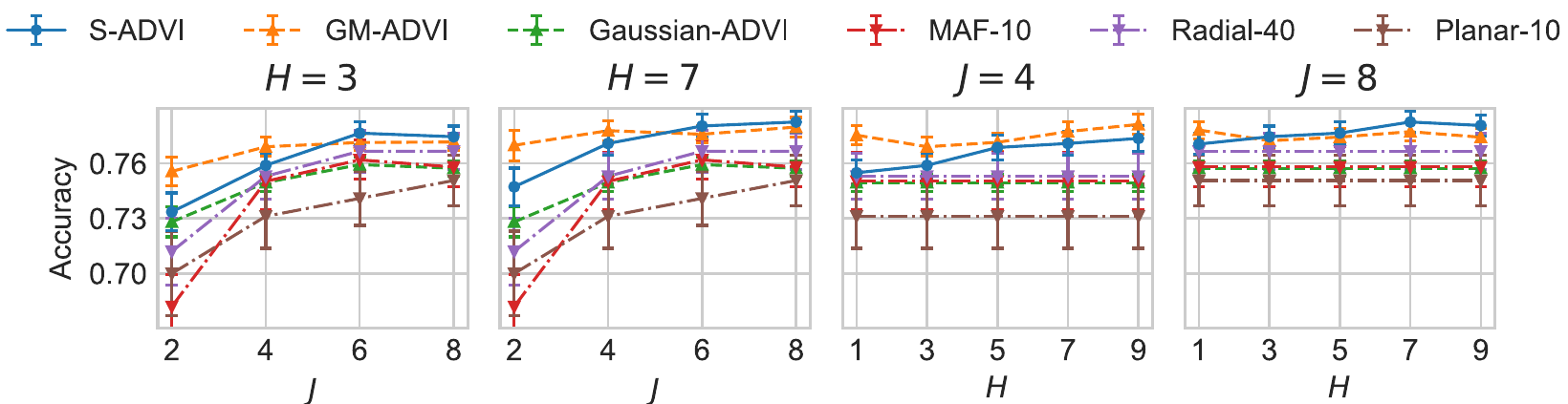}
    \vskip -0.1in
	\caption{\capitalisewords{Performance comparison for single column classification for the MNIST dataset with a range of latent variables $J$ and interior knots $H$ (Error bars denote the standard error)}} \label{fig:performace-supp}
\end{figure*}

\begin{figure}[!htbp]
    \centering\includegraphics[width=0.4\textwidth]{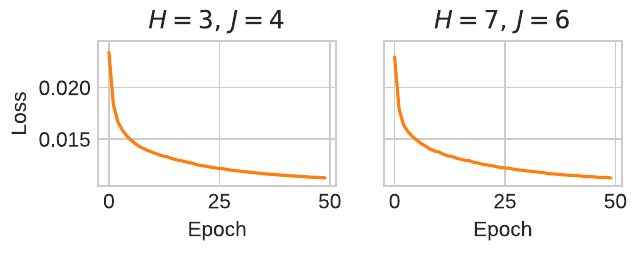}
    \caption{\capitalisewords{Convergence of S-ADVI in terms of IWAE loss for FMNIST}}
    \label{fig:fminst_loss}
\end{figure}

To evaluate the influence of hyperparameters, we also test the prediction performance of our method when $J>8$ and the number of basis spline functions is larger. For latent dimensions $J > 8$, we observe similar performance with average accuracy 0.683 at $J = 10$, 0.686 at $J = 12$, and 0.684 at $J = 14$ for S-ADVI with $H = 3$. In comparison, GM-ADVI achieves 0.679, 0.683, 0.680, and Radial achieves 0.667, 0.673, 0.675, respectively. Given the marginal performance
gains with additional latent dimensions, these results were omitted. While penalized spline mitigates overfitting,
excessive spline bases lower performance (accuracy drops from 0.680 at $H = 15$ to 0.674 at $H = 30$).

\subsection{Comparison of Two Annealing Methods} \label{app_sec:annealing_exp}

\begin{figure}[!htbp]
    \centering\includegraphics[width=0.8\textwidth]{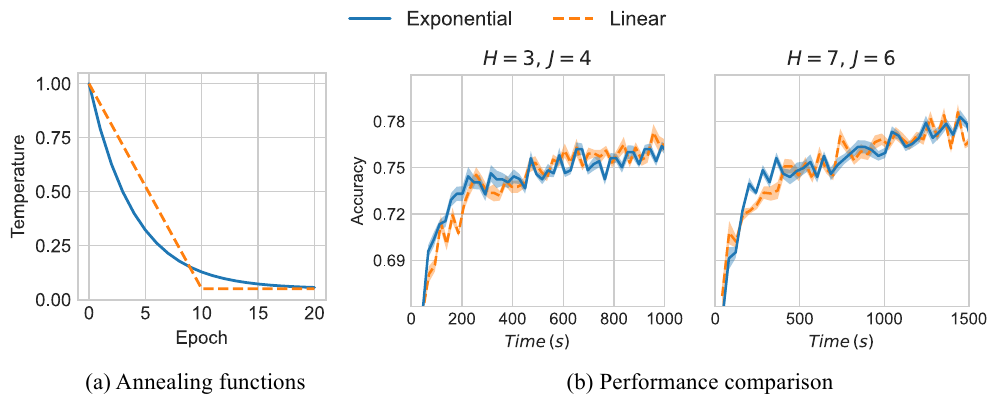}
    \caption{(a) \capitalisewords{Illustration of two annealing functions} (b) \capitalisewords{Performance comparison for two annealing functions under different combinations of $H$ and $J$} (The shadow denote the standard error)}
    \label{fig:annealing_supp}
\end{figure}

Annealing is important in the proposed method to ensure most of the combinations of coefficients are explored. In this section, we evaluate the influences of annealing functions on model performance. We consider two annealing functions in this experiment:
\begin{description}
  \item[Exponential] $\Lambda(c)=\Lambda_1+ (\Lambda_0-\Lambda_1)e^{-c/\eta}$,
  \item[Linear] 
  \begin{equation*}
    \Lambda(c)=
    \begin{cases}
      \Lambda_0 - \frac{\Lambda_0 - \Lambda_1}{\eta'}c & \text{if $c \le \eta'$}\\
      \Lambda_1 & \text{if $c > \eta'$}    \end{cases}.
  \end{equation*}
\end{description}
The parameter $\eta$ in the exponential annealing function controls the decreasing rate of the temperature, while $\eta'$ in the linear annealing function determines the decreasing rate and the turning point where the temperature becomes a constant.

We run the same single-column classification task on MNIST as in Section \ref{SEC:Single-Column}. The effects of annealing functions are evaluated based on the convergence speed. As in the main paper, we consider two combinations of $H$ and $J$: ($H=3,\, J=4$) and ($H=7,\, J=6$), and choose the classification accuracy as the performance metric. For the exponential function, we use $\Lambda_0 = 1,\, \Lambda_1 = 0.05$ and $\eta = 4$. For the linear function, we use $\Lambda_0 = 1,\, \Lambda_1 = 0.05$ and $\eta' = 10$.

The results are shown in Figure \ref{fig:annealing_supp}. We find that the exponential annealing function used in the main text converges faster than the linear function under both parameter combinations. On the other hand, after the initial epochs, the performance is similar. One potential reason is that the temperature decreases faster with the exponential function at the first few epochs, which helps the model to better approximate the categorical distribution with better parameter estimation. 

\subsection{Imaging Reconstructions} \label{app_sec:full_image_exp}

\begin{figure}[!htb] 
	\centering
    \subfloat[MNIST]{\includegraphics[width=\textwidth]{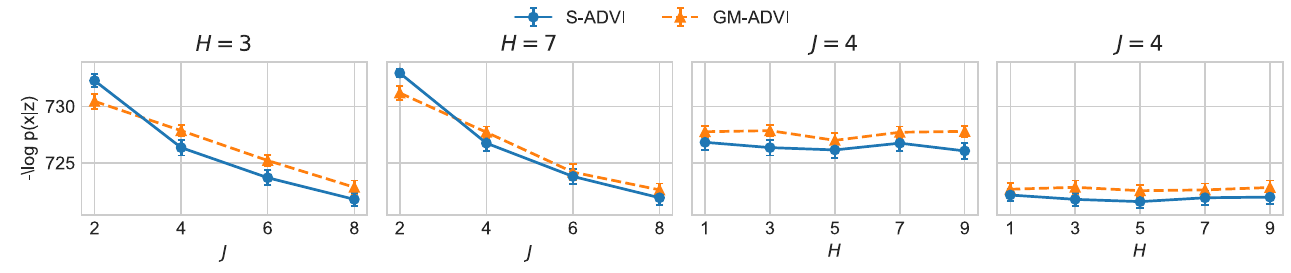}}\\
  	\subfloat[FMNIST]{\includegraphics[width=\textwidth]{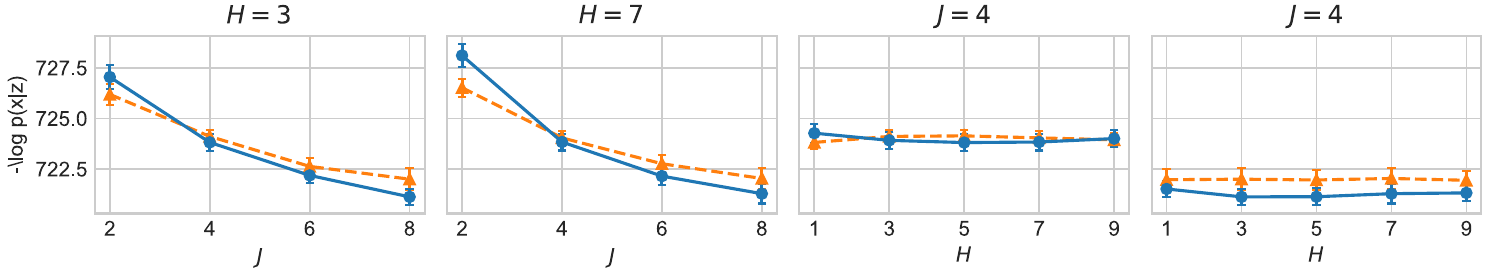}}\\
	\subfloat[CIFAR-10]{\includegraphics[width=\textwidth]{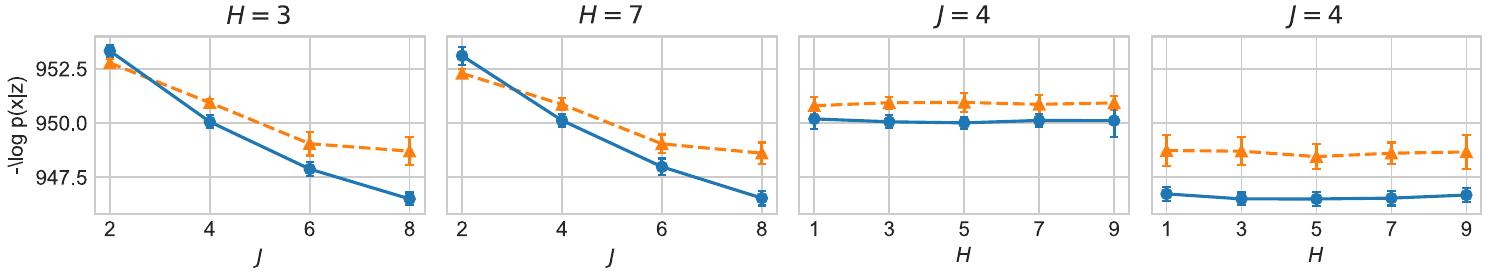}}\\
	\caption{\capitalisewords{Reconstruction comparison for full image VAE with a range of latent variables $J$ and interior knots $H$ (smaller is better). Error bars denote the standard error.}} \label{fig:recon_performance}
\end{figure}

Our previous experiments suggest that the proposed S-ADVI methods can match and outperform other VI-based methods in terms of classification task, even when the data does not contain sufficient information. In this experiment, we compare the S-ADVI and GM-ADVI methods in terms of reconstruction using MNIST \citep{lecun1998gradient}, FMNIST \citep{xiao2017online}, and CIFAR-10 \citep{krizhevsky2009learning} datasets. 

The network structures of the encoder and decoder are MLP with 2 layers. The numbers of hidden units are 512 and 256 for the encoder, and 256 and 512 for the decoder.  We choose $\{0.25,\, 0.50,\, 0.75\}$ as the interior knots and use 4 latent variables. We set $T=10$ for IWAE, use isotropic Gaussian distribution as the prior, and train the model with Adam optimizer with a learning rate of 0.001 decayed by 5\% for every epoch.  The final output of the decoder includes the means and standard deviations of the pixels. For the CIFAR-10 dataset, we transform the color images to greyscales. We evaluate the reconstruction performance using $-\log p(\bm x|\bm z)$, where we assume that $p(\bm x|\bm z)$ follows Gaussian distribution. 

The results are shown in Figure \ref{fig:recon_performance}. We find that when the number of latent variables is limited, GM-ADVI outperforms S-ADVI. When $J$ increases, the performance of both models improves, and S-ADVI outperforms GM-ADVI. The number of bases $H$ also influences model performance when $H$ increases from 1 to 3, but further increasing $H$ does not improve model performance much and can lead to overfitting. It is also interesting to note that the advantage of S-ADVI increases for more complicated datasets (e.g., MNIST vs CIFAR-10) with higher $J$ and $H$. Several facts may contribute to the observation. For example, S-ADVI is capable of capturing bounded and skewed posterior distribution. The pre-specified spline bases can help prevent the model from concentrating on the few most important modes and ignoring those relatively less important modes. The bounded support of the posterior from S-ADVI may mitigate overfitting. 

\subsection{Analysis of the Shape of the Approximated Posterior with VAE}
\label{SEC:shape-supp}

The network structures of the encoder and decoder are MLP with 2 layers. The numbers of hidden units are 512 and 256 for the encoder, and 256 and 512 for the decoder. We use $\beta = 0.05$ penalty on the KL divergence term. We choose $\{0.25,\, 0.50,\, 0.75\}$ as the interior knots and use 4 latent variables. We set $T=10$ for IWAE, use isotropic Gaussian distribution as the prior, and train the model with Adam optimizer with a learning rate of 0.001 decayed by 5\% for every epoch. 


Examples of the approximated $q_{\phi}(\epsilon_j | \bm x)$ are illustrated in Figure \ref{fig:latent-pdf} for digits 0, 2, and 4. We randomly select 100 samples for each digit from the testing set and visualize the approximated posteriors' probability density functions (PDFs) and their average values across all samples. 
The results show that the proposed method can effectively capture the skewness (e.g., $q_{\phi}(\epsilon_3 | \bm x)$) and multimodality (e.g., $q_{\phi}(\epsilon_1 | \bm x)$) of the posterior distribution. 
Furthermore, in general, while labels are not included in the training process, the approximated PDFs of samples of the same digit share similar shapes, whereas samples of different digits can have different shapes. For example, samples of digits 0 and 2 have different PDFs for $\epsilon_3$, and the modes differ for $\epsilon_4$ for samples of digits 0, 2, and 4. These observations imply that the shape of the PDFs of the underlying posteriors can capture the similarities of samples from the same group (samples of the same digit) and the differences of samples from different groups. 
Figure \ref{fig:latent-pdf_2} shows original samples and S-VAE reconstructions for digit 2, following the same format as Figure~\ref{fig:latent_individual_ig} in the main text. Similar patterns are observed, where bases with higher coefficients represent larger image regions and capture key image characteristics.

\begin{figure*}[!t] 
    \centering
    \includegraphics[width=\textwidth]{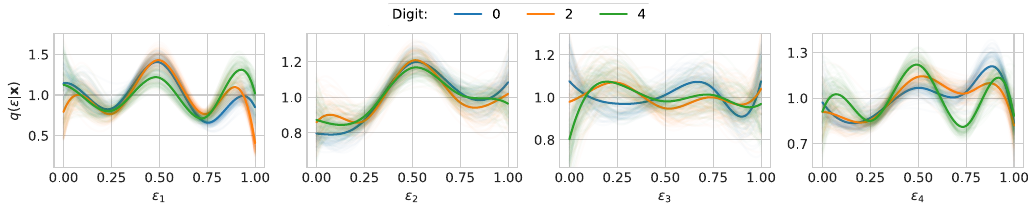}
    \vskip -0.1in
	\caption{\capitalisewords{Illustration of shapes of latent variable distributions for digits 0, 2, and 4. We randomly select 100 samples for each of the digits. PDFs of individual samples are in light colors, and the deep solid curve denotes the averages of PDF values}} \label{fig:latent-pdf}
\end{figure*}

\begin{figure*}[!t] 
    \centering
    \includegraphics[width=\textwidth]{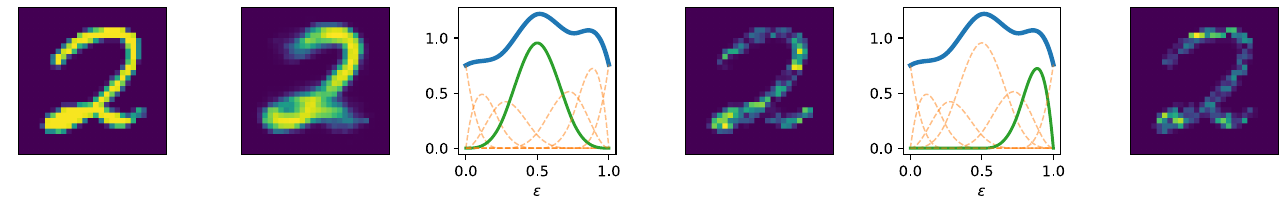}
    \vskip -0.12in
	\caption{\capitalisewords{Analysis of the approximated shape of posterior from S-VAE for digit 2}} \label{fig:latent-pdf_2}
\end{figure*}

\begin{figure}[!htb] 
	\centering
    \subfloat[Digit 1]{\includegraphics[width=.975\textwidth]{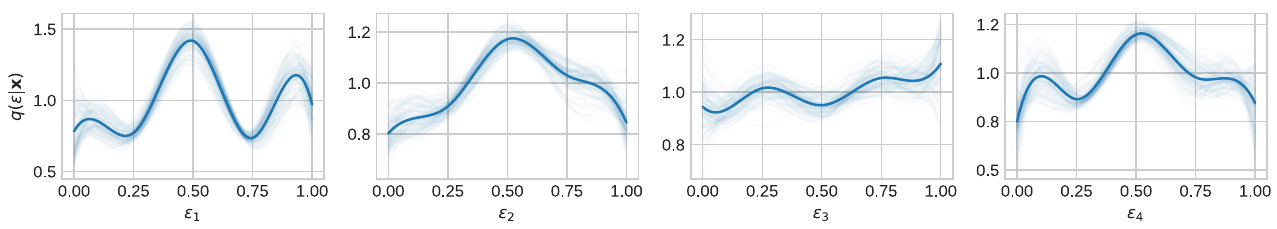}}\\
  	\subfloat[Digit 3]{\includegraphics[width=.975\textwidth]{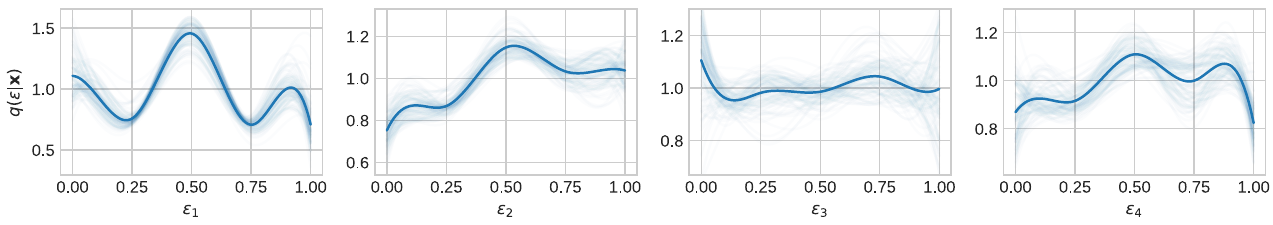}}\\
	\subfloat[Digit 5]{\includegraphics[width=.975\textwidth]{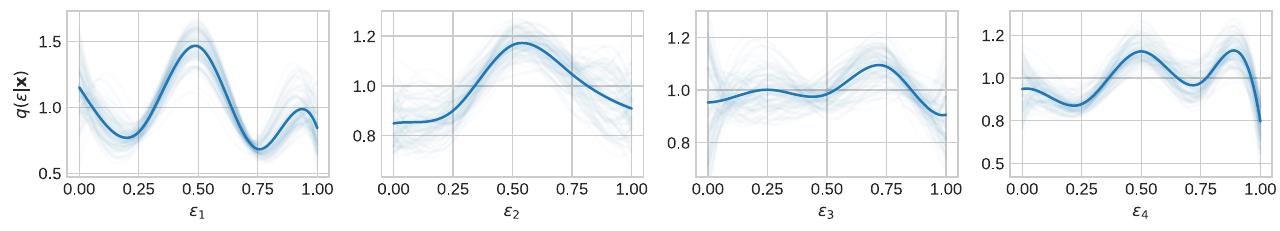}}\\
    \subfloat[Digit 6]{\includegraphics[width=.975\textwidth]{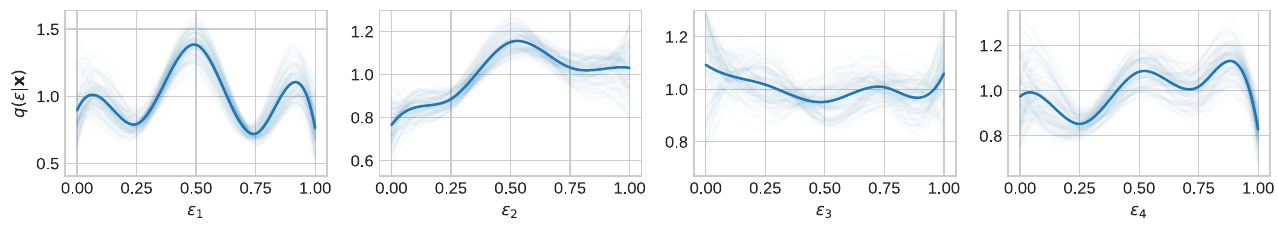}}\\
	\caption{\capitalisewords{Visualization of latent distributions for digits 1, 3, 5, 6}} \label{fig:latent_individual_pdf_1356}
\end{figure}

\begin{figure}[!htb] 
	\centering
    \subfloat[Digit 7]{\includegraphics[width=.975\textwidth]{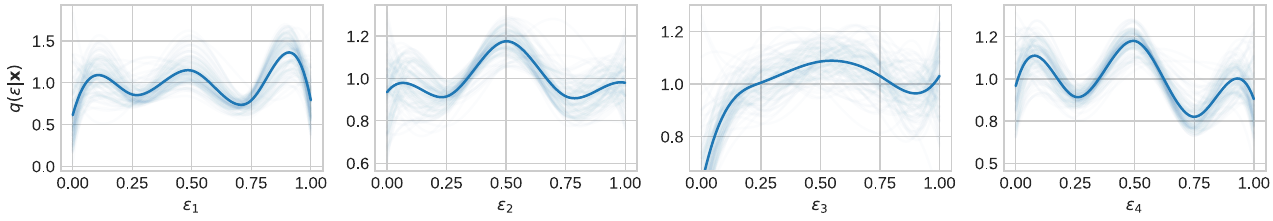}}\\
    \subfloat[Digit 8]{\includegraphics[width=.975\textwidth]{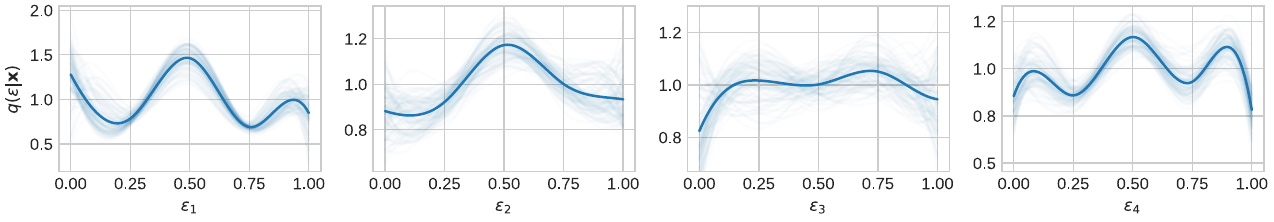}}\\
    \subfloat[Digit 9]{\includegraphics[width=.975\textwidth]{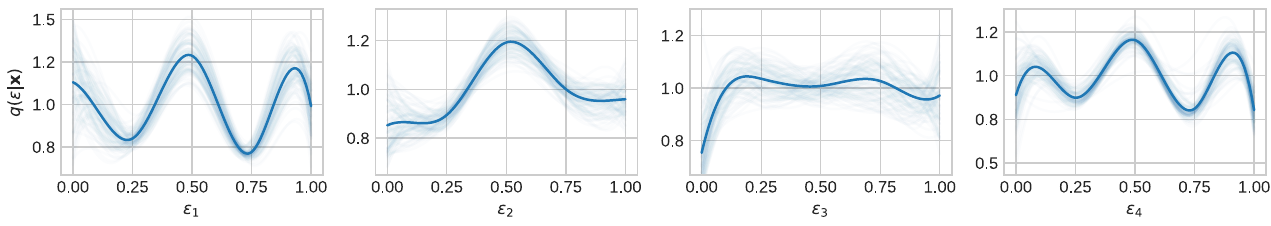}}\\
	\caption{\capitalisewords{Visualization of latent distributions for digits 7, 8, 9}} \label{fig:latent_individual_pdf_789}
\end{figure}

In Figures \ref{fig:latent_individual_pdf_1356} and \ref{fig:latent_individual_pdf_789}, we visualize the latent variables for digits 1, 3, 5, 6, 7, 8, 9. We find that the shapes of posteriors of certain latent variables can be shared across digits with similar patterns, such as $\epsilon_1$ for digits 3, 5, 8, and 9. We also observe that for some digits (e.g., 5, 6, 8, 9), the shapes of PDF of $\epsilon$ are inconsistent within the groups. One possible reason is the existence of pattern variation in the handwritten digits, i.e., samples of the same digit can be further divided into subgroups, where samples from different subgroups can have very different handwritten patterns. As discussed in Section \ref{sec:discussion}, further evaluation and investigation with statistical methods are required to identify and interpret the observed inconsistent PDFs.


We visualize the connections between input samples and the shapes of the approximated $q_\phi(\epsilon_j| \bm x)$ for two digits, 6 and 7, in Figure \ref{fig:individual_samples_67}. Based on Figures \ref{fig:latent_individual_pdf_1356} and \ref{fig:latent_individual_pdf_789}, we focus on the 4th basis (the centermost spline basis) of $\epsilon_1$ and the 6th basis of $\epsilon_4$ (the second to the right spline basis). We randomly select three samples for each of the digit. We find that the results are consistent with the examples in the main text for digits 0 and 4. Spline bases with higher weights are usually associated with larger regions in the images. We also observe that the same basis usually points to the same digit region. For example, IG suggests that basis 4 of $\epsilon_1$ is associated with the middle part of digit 6, and basis 6 of $\epsilon_4$ is associated with the lower part of digit 7.

\begin{figure}[!tbp] 
	\centering
    \subfloat[Digit 6, basis 4 of $\epsilon_1$]{\includegraphics[width=0.45\textwidth]{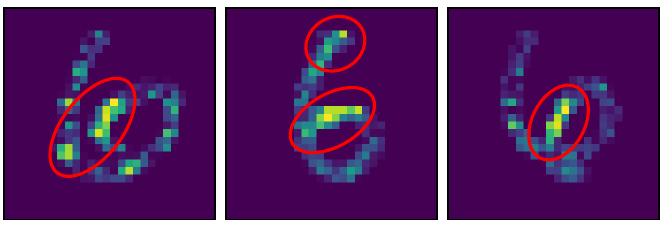}} \hspace{0.9ex}
    \subfloat[Digit 6, basis 6 of $\epsilon_4$]{\includegraphics[width=0.45\textwidth]{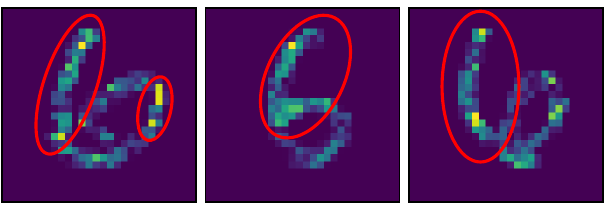}}\\
    \subfloat[Digit 7, basis 4 of $\epsilon_1$]{\includegraphics[width=0.45\textwidth]{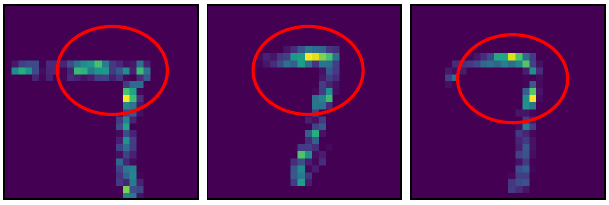}} \hspace{0.9ex}
    \subfloat[Digit 7, basis 6 of $\epsilon_4$]{\includegraphics[width=0.45\textwidth]{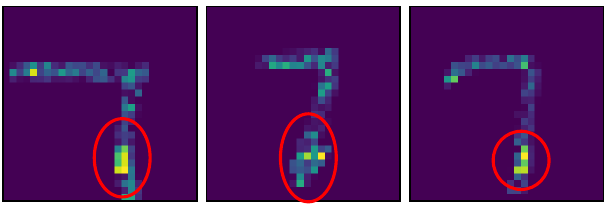}}\\
	\caption{\capitalisewords{Interpretation of weighted spline bases for samples of digits 6 and 7. Brighter pixels correspond to higher absolute attributes. Regions containing highly attributed pixels are highlighted in red circles}} \label{fig:individual_samples_67}
\end{figure}
\FloatBarrier

\small

\bibliographystyle{asa}
\bibliography{references_TF}

\end{document}